%% file: submit_for_arxiv.tex
\documentclass{article}
\PassOptionsToPackage{table}{xcolor}
\usepackage[main,final]{neurips_2026}
\usepackage[utf8]{inputenc}
\usepackage[T1]{fontenc}
\PassOptionsToPackage{table}{xcolor}
\input{packages.tex}

\input{config/acronyms.tex}

\input{config/math.tex}

\usepackage[capitalize,noabbrev]{cleveref}
\theoremstyle{plain}
\newtheorem{theorem}{Theorem}[section]

\theoremstyle{definition}

\theoremstyle{remark}

\title{Flow-based Spectral Kernel Learning for Nonstationary Attention}
\input{authors}
\hypersetup{pdftitle={Flow-based Spectral Kernel Learning for Nonstationary Attention},pdfauthor={Zi Yang, Ying Li, Alejandro Lancho, Dario C. Larese, Pablo M. Olmos, Michael Minyi Zhang}}
\begin{document}
\maketitle
\begin{abstract}

Kernelized attention provides a useful perspective for understanding self-attention as a similarity-based aggregation mechanism, but existing approaches often rely on fixed or stationary kernels, limiting their ability to adapt the similarity structure to complex token interactions.   
In this paper, we propose flow-based spectral kernel attention (FSKA), a framework for learning expressive nonstationary attention kernels.
FSKA formulates kernelized attention through spectral density learning and parameterizes the resulting bivariate spectral density with normalizing flows.
This design enables flexible and regularizable kernel learning while remaining compatible with scalable linear-attention computation. 
We further show that expressive kernel learning enables a simplified attention design with removed query--key projections and fixed orthogonal value projections, resulting in a more parameter-efficient architecture.
Controlled experiments on classification, ablation, and scalability profiling show that FSKA achieves competitive or improved performance over existing attention mechanisms, while reducing projection-related trainable parameters by approximately $40\%$--$47\%$ compared with most baselines and maintaining favorable linear-attention memory scaling.
Code Available: \url{https://github.com/ziyang18/FSKA}.
\end{abstract}
\input{context/introduction}
\input{context/background}

\input{context/method}
\input{context/experiment}
\label{main_text_end}
\clearpage
\begin{ack}
Ying Li acknowledges support from Singapore MOE AcRF Tier 2 grants MOE-T2EP20123-0010 and MOE-T2EP20223-0011.

Alejandro Lancho has received funding from the Comunidad de Madrid's 2023 Cesar Nombela program under Grant Agreement No.~2023-T1/COM-29065. Additionally, he has received funding from the Ministerio de Ciencia, Innovaci\'on y Universidades, Spain, under Grant Agreement No.~PID2023-148856OA-I00 (funded by MICIU/AEI/10.13039/501100011033 and by ERDF/UE).

Pablo M. Olmos has been supported by the Ministerio de Ciencia, Innovaci\'on y Universidades, Spain, under Grant PID2024-157856NB-I00 (CARTESIAN), funded by MCIN/AEI/\allowbreak 10.13039/501100011033 and, as appropriate, by ``ERDF A way of making Europe''; by the ELLIS Unit Madrid (European Laboratory for Learning and Intelligent Systems); and by the Comunidad de Madrid under grant IND2024/TIC-34728 and the IDEA-CM project (\mbox{TEC-2024/COM-89}).

The contribution of Michael Zhang is funded by the Research Grants Council under the General Research Fund \#17312826 and by the University of Hong Kong Seed Fund for PI Research -- Basic Research \#2402101367.
\end{ack}

\subsection*{AI Use Statement}
Generative AI tools were used to polish the language of this manuscript.
\subsection*{Reproducibility Statement}
Mathematical assumptions and proofs are provided in Appendix~\ref{app:Non_stationary_RFF} and Appendix~\ref{app:proof_qk_absorption}. Experimental descriptions, data preparation, configuration details, and supplementary comparisons are collected in Appendix~\ref{app:additional_experiments}, including document classification in Appendix~\ref{app:document_protocol}.
Code Available: \url{https://github.com/ziyang18/FSKA}.
\bibliography{ref}
\bibliographystyle{plainnat}
\clearpage
\appendix
\input{context/app}

\end{document}

%% file: packages.tex
\usepackage[acronym,nowarn]{glossaries}
\glsdisablehyper
\makeglossaries
\usepackage{xspace}
\usepackage{float}
\usepackage{physics}
\usepackage{lipsum}

\usepackage{enumitem}

\usepackage{amssymb}
\usepackage{mathtools}
\usepackage{amsfonts}
\usepackage{amsmath}
\usepackage{amsthm}
\usepackage{booktabs}
\usepackage[]{microtype}

\usepackage{pifont}
\newcommand{\cmark}{\textcolor{green!60!black}{\ding{51}}\xspace}
\newcommand{\xmark}{\textcolor{red!60!black}{\ding{55}}\xspace}

\usepackage[vlined,linesnumbered,ruled]{algorithm2e}

\SetCommentSty{mycommfont}
\SetKwInput{KwInput}{Input} 
\SetKwInput{KwOutput}{Output}

\usepackage{xcolor}
\usepackage{colortbl}
\definecolor{mylightgray}{gray}{0.94}

\definecolor{color_vae}{HTML}{91058A}
\definecolor{color_cvae}{HTML}{DA70D6}
\definecolor{color_gppvae}{HTML}{32CD32}
\definecolor{color_gpvae}{HTML}{039796}
\definecolor{color_gp_vae}{HTML}{00FFFF}
\definecolor{color_sgp_vae}{HTML}{99CB32}
\definecolor{color_svgp_vae}{HTML}{19198c}
\definecolor{color_bae}{HTML}{AB6345}
\definecolor{color_gpbae}{HTML}{FF9B38}
\definecolor{color_bsgpae}{HTML}{DF2B4F}
\definecolor{color_blue_trajectory}{HTML}{DF2B4F}
\definecolor{color_blue_trajectory}{HTML}{3E9EFF}
\definecolor{color_orange_trajectory}{HTML}{DF5E33}

\usepackage[colorlinks,linktoc=all]{hyperref}
\usepackage[all]{hypcap}
\definecolor{frenchblue}{rgb}{0.01171875, 0.0078125, 0.4375}
\hypersetup{citecolor=frenchblue}
\hypersetup{linkcolor=frenchblue}
\hypersetup{urlcolor =frenchblue}

\makeatletter
\@ifundefined{c@rownum}{%
  \let\c@rownum\rownum
}{}
\@ifundefined{therownum}{%
  \def\therownum{\@arabic\rownum}%
}{}
\makeatother

\makeatletter
\newcommand*{\addFileDependency}[1]{%
	\typeout{(#1)}
	\@addtofilelist{#1}
	\IfFileExists{#1}{}{\typeout{No file #1.}}
}
\makeatother

\usepackage[font=small,labelfont=bf,tableposition=top]{caption}
\usepackage{tikz}
\usetikzlibrary{arrows.meta, positioning, calc}
\usepackage{graphicx}
\usepackage[]{subcaption}   %

\usepackage{pgfplots}
\pgfplotsset{compat=1.6}
\usepgfplotslibrary{groupplots}
\tikzstyle{every picture}+=[font=\sffamily]
\tikzstyle{optimized} = [circle,fill=white,draw=black, dashed,inner sep=1pt, minimum size=20pt, font=\fontsize{10}{10}\selectfont, node distance=1]
\pgfkeys{/pgf/number format/.cd,1000 sep={}}
\pgfplotsset{
	tick label style = {font=\sffamily},
	every axis label/.append style={font=\sffamily},
	typeset ticklabels with strut,
}
\pgfplotsset{every axis/.append style={
			every x tick label/.append style={font=\fontsize{6pt}{6pt}\sffamily, yshift=.5ex,},
			every y tick label/.append style={font=\fontsize{6pt}{6pt}\sffamily, xshift=.5ex},
			every y label/.append style={xshift=10ex, font=\sffamily},
			every x label/.append style={yshift=3ex, font=\sffamily},
			every title/.append style={font=\sffamily}
		},
}
\pgfplotsset{
  xticklabel={$\mathsf{\pgfmathprintnumber{\tick}}$},
  yticklabel={$\mathsf{\pgfmathprintnumber{\tick}}$},
}
\pgfplotsset{every axis title/.append style={yshift=-1ex}}
\newlength\figureheight
\newlength\figurewidth
\usepgfplotslibrary{external}
\usepackage[most]{tcolorbox}
{\begin{tcolorbox}[toprule=2mm,left=4pt,right=4pt,top=0pt,bottom=1pt,boxsep=2pt, before skip = 2ex, after skip =2ex]}{\end{tcolorbox}}
\tcbset{boxsep=4pt,left=2pt,right=2pt,top=-0pt,bottom=0pt}

\usepackage{url}

\usepackage{xr}

\usepackage{subcaption}

\usepackage{tikz}
\usepackage{graphicx}
\usepackage{xcolor}
\usepackage{amsmath}
\usepackage{amssymb}
\usepackage{bm}

\usepackage{adjustbox}
\usepackage{multirow}
\usepackage{mathtools}

\usepackage{xspace}

%% file: config/acronyms.tex
\newacronym{MAP}{map}{maximum-a-posteriori}
\newacronym{MLE}{mle}{maximum likelihood estimation}
\newacronym{MNLL}{mnll}{mean negative loglikelihood}
\newacronym{NLL}{nll}{negative loglikelihood}
\newacronym{LL}{ll}{log-likelihood}
\newacronym{RMSE}{rmse}{root mean squared error}
\newacronym{ECE}{ece}{expected calibration error}
\newacronym{FID}{fid}{Fr\'echet Inception Distance}
\newacronym{MSE}{mse}{mean squared error}

\newacronym{PCA}{pca}{principal component analysis}
\newacronym{AE}{ae}{autoencoder}
\newacronym{WAE}{wae}{Wasserstein Autoencoder}
\newacronym{VAE}{vae}{Variational Autoencoder}
\newacronym{BAE}{bae}{Bayesian autoencoder}
\newacronym{CDF}{cdf}{cumulative density function}
\newacronym{GAN}{gan}{Generative Adversarial Network}

\newacronym{MC}{mc}{Monte Carlo}
\newacronym{MCMC}{mcmc}{Markov chain Monte Carlo}
\newacronym{HMC}{hmc}{Hamiltonian Monte Carlo}
\newacronym{MH}{mh}{Metropolis-Hastings}
\newacronym{NUTS}{nuts}{no-u-turn sampler}
\newacronym{SGHMC}{sghmc}{stochastic gradient Hamiltonian Monte Carlo}

\newacronym{DGP}{dgp}{deep Gaussian process} %
\newacronym{GPLVM}{gplvm}{Gaussian process latent variable model}
\newacronym{DPMM}{dpmm}{Dirichlet Process Mixture Model}

\newacronym{VFE}{vfe}{variational free energy}

\newacronym[firstplural=Gaussian Processes]{GP}{gp}{Gaussian Process}

\newacronym{VI}{vi}{variational inference}

\newacronym{ELBO}{elbo}{evidence lower bound}
\newacronym{NELBO}{nelbo}{negative evidence lower bound}
\newacronym{ELL}{ell}{expected log likelihood}
\newacronym{KL}{kl}{Kullback-Leibler divergence}
\newacronym{AUC}{auc}{area under the curve}

\newacronym[firstplural=Bayesian neural networks]{BNN}{bnn}{Bayesian neural network}
\newacronym[firstplural=deep neural networks]{DNN}{dnn}{deep neural network}
\newacronym[]{CNN}{cnn}{convolutional neural network}
\newacronym{MLP}{mlp}{multilayer perceptron}
\newacronym{NN}{nn}{neural network}
\newacronym{RELU}{ReLU}{rectified linear unit}

\newacronym{NF}{nf}{normalizing flow}

\newacronym{RBF}{rbf}{radial basis function}
\newacronym{ARD}{ard}{automatic relevance determination}

\newacronym{RKHS}{rkhs}{reproducing kernel Hilbert space}

\newacronym{OT}{ot}{optimal transport}
\newacronym{WD}{wd}{Wasserstein distance}
\newacronym{SWD}{swd}{sliced-Wasserstein distance}
\newacronym{DSWD}{dswd}{distributional sliced-Wasserstein distance}
\newacronym{BSGPAE}{bsgpae}{Bayesian Sparse Gaussian Process Autoencoder}
\newacronym{GPBAE}{{gp}-{bae}}{Gaussian Process Bayesian Autoencoder}
\newacronym{CVAE}{cvae}{Conditional Variational Autoencoder}
\newacronym{SGPBAE}{{sgp}-{bae}}{Sparse Gaussian Process Bayesian Autoencoder}

%% file: config/math.tex


\DeclarePairedDelimiterX{\infdivx}[2]{[}{]}{%
  #1\;\delimsize\|\;#2%
}



%% file: authors.tex
\author{%
{\mdseries Zi Yang$^{1,*}$\quad Ying Li$^{2,*}$\quad Alejandro Lancho$^3$}\\
{\mdseries Dar\'{i}o C. Larese$^3$\quad Pablo M. Olmos$^{3,\dagger}$\quad Michael Minyi Zhang$^{1,\dagger}$}\\[4pt]
$^1$School of Computing and Data Science, The University of Hong Kong\\
$^2$Department of Statistics and Data Science, National University of Singapore\\
$^3$Department of Signal Theory and Communications, Universidad Carlos III de Madrid\\[3pt]
\texttt{ziyang2023@connect.hku.hk}\quad\texttt{ying-li@nus.edu.sg}\\
\texttt{alancho@ing.uc3m.es}\quad\texttt{dcabezas@tsc.uc3m.es}\\
\texttt{pamartin@ing.uc3m.es}\quad\texttt{mzhang18@hku.hk}\\[3pt]
{\small $^*$Equal contribution.\quad $^\dagger$Corresponding authors.}
}

%% file: context/introduction.tex
\section{Introduction}
Transformers have become the dominant architecture for sequence modeling in natural language processing, vision, and multimodal learning, largely due to the success of self-attention \citep{vaswani2017attention,brown2020language,dosovitskiy2020image}. At its core, self-attention computes pairwise token similarities and uses them to aggregate information. This similarity-based view has motivated a growing line of work that interprets attention as a kernel machine \citep{tsai2019transformer}. Under this kernel perspective, different kernel choices induce different similarity measures for token aggregation, with standard softmax attention corresponding to one particular kernel.

This kernel view has led to substantial progress on efficient attention. Methods such as random feature attention (RKA) and the Performer approximate the softmax kernel using random features, enabling linear-time attention with respect to sequence length \citep{peng2021random,choromanski2020rethinking}. However, these approaches typically rely on fixed or stationary kernels, whose similarity measures do not adapt to the input data. As a result, they are often limited in modeling complex or context dependent (nonstationary) interactions between tokens.

To overcome the rigidity of fixed kernels, several recent works have explored learnable kernel attention under the Bochner's theorem.
These methods learn the spectral density of the kernel and approximate the resulting kernel using random Fourier features (RFFs) for scalable attention.
For example, GMM-RKS parameterizes the spectral density with Gaussian mixtures, while implicit kernel attention (IKAN) learns the spectral density through neural-network (NN) \citep{pal2022learning,song2021implicit,li2019implicit}. Although these methods improve flexibility, GMM-RKS can become costly when many components are needed to model complex spectra, and IKAN lacks an explicit density for principled spectral regularization and may suffer from overfitting or spectral collapse (see Figure~\ref{fig:method_spectra}).  In summary, existing approaches often fall between one of two extremes: fixed kernels are efficient but too restrictive, whereas highly flexible kernels are more expressive but often lack explicit spectral densities for regularization or can become computationally burdensome.

In this paper, we address this gap by proposing ``flow-based spectral kernel attention'' (FSKA), a flexible and efficient framework for learning nonstationary attention kernels.
Our main contributions are summarized as follows:
\begin{itemize}
    \item We formulate nonstationary attention kernel learning as spectral density learning.
    To obtain flexible and stable kernels, we parameterize the spectral density with normalizing flows and regularize it with a structured spectral prior.
    The resulting kernel remains compatible with RFF-based linear attention.
    \item We show that expressive kernel learning enables a simplified attention design with removed query--key projections and fixed orthogonal value projections, resulting in a more parameter-efficient architecture.

    \item We conduct experiments across regression and classification tasks, together with ablation and scalability studies. FSKA achieves competitive or improved performance across these settings, reduces projection-related trainable parameters by approximately $40\%$--$47\%$ compared with most baselines, and maintains a low-growth memory profile close to linear-attention baselines.
\end{itemize}

%% file: context/background.tex
\section{Background}
This section briefly reviews self-attention in Section~\ref{sec:self-attn}, its kernel interpretation in Section~\ref{sec:kernel-attn}, and RFFs for scalable kernel attention in Section~\ref{sec:rff}.

\subsection{Self-Attention}
\label{sec:self-attn}
Given an input sequence 
$\mathbf{X} = [\mathbf{x}_1, \dots, \mathbf{x}_L] \in \mathbb{R}^{L \times d_x}$ 
consisting of $L$ tokens with input embedding dimension $d_x$, 
the self-attention mechanism~\citep{vaswani2017attention} computes pairwise interactions through learned similarity scores.
Specifically, the inputs are linearly projected into queries, keys, and values:
\begin{equation}
    \mathbf{Q} = \mathbf{X} \mathbf{W}_Q,\qquad
    \mathbf{K} = \mathbf{X} \mathbf{W}_K,\qquad
    \mathbf{V} = \mathbf{X} \mathbf{W}_V,
\end{equation}
where 
$\mathbf{W}_Q \in \mathbb{R}^{d_x \times d_q}$,
$\mathbf{W}_K \in \mathbb{R}^{d_x \times d_k}$, and
$\mathbf{W}_V \in \mathbb{R}^{d_x \times d_v}$
are learnable projection matrices.
For dot-product attention, we take $d_q=d_k$, so that 
$\mathbf{Q}\mathbf{K}^{\top}$ is well-defined.
The attention weights are computed by
\begin{equation}
    \mathbf{A} = \operatorname{softmax}\!\left(\mathbf{Q} \mathbf{K}^\top / \sqrt{d_k}\right),
\end{equation}
and the output is given by $\mathbf{Y} = \mathbf{A} \mathbf{V}$.
Multi-head self-attention (MHSA) extends this formulation by computing attention over $H$ heads with independent projections and concatenating their outputs.
Overall, self-attention can be viewed as a learned similarity mechanism between tokens.

\subsection{Kernel View of Attention}
\label{sec:kernel-attn}
A key observation by \citet{tsai2019transformer} is that self-attention admits a kernel interpretation.
Given a non-negative kernel function $\kappa(\cdot, \cdot)$, the generalized attention output at position $i$ can be written as
\begin{equation}
\mathrm{Attn}(\mathbf{X})_i
=
\frac{
\sum_{j=1}^{L} \kappa(\mathbf{q}_i, \mathbf{k}_j)\, \mathbf{v}_j
}{
\sum_{j'=1}^{L} \kappa(\mathbf{q}_i, \mathbf{k}_{j'})
},
\label{eq:kernel_attn}
\end{equation}
where $\mathbf{q}_i = \mathbf{x}_i \mathbf{W}_Q$, $\mathbf{k}_j = \mathbf{x}_j \mathbf{W}_K$, and $\mathbf{v}_j = \mathbf{x}_j \mathbf{W}_V$.
When $\kappa(\mathbf{q}, \mathbf{k}) = \exp(\mathbf{q} \mathbf{k}^\top / \sqrt{d_k})$, Eq.~\eqref{eq:kernel_attn} reduces to standard softmax attention.
This formulation reveals that attention is fundamentally a kernel smoother operating on token representations.
However, computing Eq.~\eqref{eq:kernel_attn} naively requires $\mathcal{O}(L^2)$ time and memory, which limits scalability to long sequences \citep{katharopoulos2020transformers}.

\subsection{Random Fourier Features for Kernel Attention}
\label{sec:rff}
Random Fourier features (RFF)~\citep{rahimi2007random} provide a scalable approximation to shift-invariant kernels by expressing them as inner products in a randomized feature space.
By Bochner's theorem, any continuous, shift-invariant kernel 
$\kappa(\mathbf{x}, \mathbf{y}) = \kappa(\mathbf{x}-\mathbf{y})$ 
can be represented as the Fourier transform of a non-negative spectral density.
In practice, RFF approximates such kernels by sampling random frequencies from the associated spectral distribution and constructing an explicit feature map.

Concretely, let 
$\boldsymbol{\phi} : \mathbb{R}^{d_{x}} \rightarrow \mathbb{R}^{2M}$ 
be a nonlinear feature map defined as
\begin{equation}
\boldsymbol{\phi}(\mathbf{x})
=
\frac{1}{\sqrt{M}}
\begin{bmatrix}
\sin(\mathbf{x} \mathbf{w}_1^\top),\,
\dots,\,
\sin(\mathbf{x} \mathbf{w}_M^\top),\,
\cos(\mathbf{x} \mathbf{w}_1^\top),\,
\dots,\,
\cos(\mathbf{x} \mathbf{w}_M^\top)
\end{bmatrix} ,
\label{eq:rff_phi}
\end{equation}
where the random vectors 
$\{\mathbf{w}_m\}_{m=1}^{M},\mathbf{w}_m \in \mathbb{R}^{1 \times d_x}$ 
are independently sampled from a spectral density $p(\mathbf{w})$ associated with the target kernel.
Under this construction, the inner product of random features provides an unbiased approximation to the kernel:
\begin{equation}
\mathbb{E}_{\mathbf{w}}
\!\left[
\boldsymbol{\phi}(\mathbf{x})^\top \boldsymbol{\phi}(\mathbf{y})
\right]
=
\kappa(\mathbf{x}, \mathbf{y}).
\label{eq:rff_expectation}
\end{equation}

This approximation allows the attention kernel to be written as
$\kappa(\mathbf{q}, \mathbf{k}) \approx \boldsymbol{\phi}(\mathbf{q})^\top \boldsymbol{\phi}(\mathbf{k})$,
which enables linear-time attention \citep{peng2021random,choromanski2020rethinking}.
However, for methods such as RKA and Performer, the random frequencies are drawn from a predetermined distribution \(p(\mathbf{w})\), so the induced spectral density cannot adapt to the data.
Moreover, the resulting kernel is stationary, depending only on $\mathbf{q}-\mathbf{k}$ rather than input-dependent patterns.
As a result, such RFF-based attention mechanisms are not sufficiently flexible to capture complex or nonstationary patterns in real-world sequences.

%% file: context/method.tex
\section{Method}
\label{sec:method}
In this section, we first introduce the nonstationary attention kernel and its random-feature approximation in Section~\ref{sec:nonstationary_kernel}.
We then propose flow-based spectral density modeling for flexible and regularized kernel learning in Section~\ref{sec:flow_spectral}.
Finally, we show how kernel expressiveness simplifies attention design by removing query--key projections, using fixed value projections in Section~\ref{sec:value_representation}, and extending the resulting module to multi-head attention in Section~\ref{sec:multi_head_kernel_attention}.

\subsection{Nonstationary Attention Kernel}
\label{sec:nonstationary_kernel}
Most existing kernelized attention methods rely on stationary kernels (e.g., $k(\mathbf{x}_1, \mathbf{x}_2) = \kappa(\mathbf{x}_1 - \mathbf{x}_2)$) derived from Bochner’s theorem,
which limits expressiveness and prevents attention from capturing input-dependent, nonstationary patterns.
An effective attention kernel should therefore be flexible enough to model complex similarities while remaining scalable with RFFs for linear-time computation.
To this end, we build on a universal Bochner theorem that establishes a duality between
continuous kernels and bivariate spectral measures \citep{yang2025multi}.
\begin{theorem}[Universal Bochner's Theorem]
\label{theo:Uni_Bochner}
A complex-valued bounded continuous kernel \( k(\mathbf{x}_1, \mathbf{x}_2) \) on \( \mathbb{R}^{d_x} \) is the covariance function of a mean square continuous complex-valued random process on \( \mathbb{R}^{d_x} \) if and only if
\begin{equation}
\label{eq:ubo}
\begin{aligned}
&k(\mathbf{x}_1,\mathbf{x}_2)
=
\frac{1}{4}
\int
\Big(
\exp\!\left(i(\mathbf{x}_1 \boldsymbol{\omega}_1^\top
-\mathbf{x}_2 \boldsymbol{\omega}_2^\top)\right)
+
\exp\!\left(i(\mathbf{x}_1 \boldsymbol{\omega}_2^\top
-\mathbf{x}_2 \boldsymbol{\omega}_1^\top)\right)
\\
&\qquad\qquad
+
\exp\!\left(i(\mathbf{x}_1 \boldsymbol{\omega}_1^\top
-\mathbf{x}_2 \boldsymbol{\omega}_1^\top)\right)
+
\exp\!\left(i(\mathbf{x}_1 \boldsymbol{\omega}_2^\top
-\mathbf{x}_2 \boldsymbol{\omega}_2^\top)\right)
\Big)
\,u(\mathrm{d}\boldsymbol{\omega}_1,\mathrm{d}\boldsymbol{\omega}_2).
\end{aligned}
\end{equation}
where \( u \) is the Lebesgue-Stieltjes measure associated with some function \( p\left(\boldsymbol{\omega}_1, \boldsymbol{\omega}_2\right) \). When \(\boldsymbol{\omega}_1 = \boldsymbol{\omega}_2\), this theorem reduces to Bochner's theorem.
\end{theorem}

Theorem~\ref{theo:Uni_Bochner} shows that learning a nonstationary continuous kernel can be viewed as learning its associated bivariate spectral density
\(p(\boldsymbol{\omega}_1,\boldsymbol{\omega}_2)\). Moreover, this formulation enables the design of flexible kernels that allow attention mechanisms to capture input-dependent patterns.

Next, we state that this duality allows us to derive a general closed-form random feature approximation for arbitrary continuous kernels, thereby enabling linear-time attention mechanisms.
\begin{theorem}
\label{theo:Non_stationary_RFF}
Let $\boldsymbol{\phi}(\mathbf{x})$ be the randomized feature map of $k(\mathbf{x}_1, \mathbf{x}_2)$, defined as follows:
\begin{equation}
\label{eq:rff_def}
\boldsymbol{\phi}(\mathbf{x})
=
\sqrt{\frac{1}{2M}}
\begin{bmatrix}
\cos \left( \mathbf{x} \boldsymbol{\omega}_{1}^{(1:M/2)\top} \right)
+ \cos \left( \mathbf{x} \boldsymbol{\omega}_{2}^{(1:M/2)\top} \right) \\
\sin \left( \mathbf{x} \boldsymbol{\omega}_{1}^{(1:M/2)\top} \right)
+ \sin \left( \mathbf{x} \boldsymbol{\omega}_{2}^{(1:M/2)\top} \right)
\end{bmatrix}
\in \mathbb{R}^{M},
\end{equation}
where the spectral frequencies 
$\{ [\boldsymbol{\omega}_1^{(m)}, \boldsymbol{\omega}_2^{(m)} ]\}_{m=1}^{M/2}
\subset \mathbb{R}^{1 \times 2d_x}$
are independent and identically distributed random row vectors drawn from the spectral density
$p(\boldsymbol{\omega}_1, \boldsymbol{\omega}_2)$.
Then, the kernel $k(\mathbf{x}_1,\mathbf{x}_2)$ admits the unbiased estimator
\begin{equation}
\label{eq:kernel_app}
k(\mathbf{x}_1, \mathbf{x}_2)
\;\approx\;
\boldsymbol{\phi}(\mathbf{x}_1)^\top  \boldsymbol{\phi}(\mathbf{x}_2) .
\end{equation}
\end{theorem}

\begin{proof}
    See Appendix~\ref{app:Non_stationary_RFF} for the full proof.
\end{proof}

Theorem~\ref{theo:Non_stationary_RFF} converts kernel evaluation into feature inner products, providing the key step toward scalable linear attention.

We next derive a nonstationary attention mechanism based on RFFs,
which is both flexible and scalable.
The nonstationary attention output at position $i$ is defined as
\begin{equation}
\label{eq:ns_attn_def}
\begin{aligned}
\mathrm{Attn}(\mathbf{X})_i
&=
\frac{
\sum_{j=1}^{L} k(\mathbf{q}_i, \mathbf{k}_j)\, \mathbf{v}_j
}{
\sum_{j'=1}^{L} k(\mathbf{q}_i, \mathbf{k}_{j'})
}
\approx
\frac{
\sum_{j=1}^{L} \boldsymbol{\phi}(\mathbf{q}_i)^\top \boldsymbol{\phi}(\mathbf{k}_j)\, \mathbf{v}_j
}{
\sum_{j'=1}^{L} \boldsymbol{\phi}(\mathbf{q}_i)^\top \boldsymbol{\phi}(\mathbf{k}_{j'})
}\\
&=\frac{
\boldsymbol{\phi}(\mathbf{q}_i)^\top 
\sum_{j=1}^{L} \big( \boldsymbol{\phi}(\mathbf{k}_j)\, \mathbf{v}_j \big)
}{
\boldsymbol{\phi}(\mathbf{q}_i)^\top 
\sum_{j'=1}^{L} \boldsymbol{\phi}(\mathbf{k}_{j'})
}.
\end{aligned}
\end{equation}
Here, $k(\cdot,\cdot)$ denotes the proposed nonstationary kernel and
$\boldsymbol{\phi}(\cdot)$ is the random feature map defined in Eq.~\eqref{eq:rff_def}.
The last equality follows by exchanging the order of the inner products and summations, which implies the attention can be computed in linear time with respect to the sequence length $L$. 

However, the sine and cosine components within the $\boldsymbol{\phi}$ in Eq.~\eqref{eq:ns_attn_def} can take both positive and negative values.
Under a finite number of random features, this may cause cancellation in the linear-attention denominator
\(\boldsymbol{\phi}(\mathbf{q}_i)^\top \sum_{j=1}^{L}\boldsymbol{\phi}(\mathbf{k}_j)\),
leading to small or even negative normalization values.
As shown in Appendix~\ref{app:prf_stability}, increasing the random feature dimension \(M\) reduces the measured fraction of negative or relatively small denominators overall, suggesting that this issue can be alleviated with more samples.
Nevertheless, near-zero denominators may still occur in practical finite-sample settings, especially when \(M\) is limited for computational efficiency.
As a simple way to address this issue, we apply a positive feature transformation before computing attention:
\begin{equation}
\label{eq:positive_feature}
\widetilde{\boldsymbol{\phi}}(\mathbf{x}) = g(\boldsymbol{\phi}(\mathbf{x})),
\qquad
g(u)=
\begin{cases}
u+1, & u>0,\\
\exp(u), & u\le 0,
\end{cases}
\end{equation}
where \(g\) corresponds to the exponential linear unit (ELU) shifted by one, i.e., \(g(u)=\mathrm{ELU}(u)+1\).
This ensures \(\widetilde{\boldsymbol{\phi}}(\mathbf{x})\geq 0\) element-wise.
We then use \(\widetilde{\boldsymbol{\phi}}\) in place of \(\boldsymbol{\phi}\) in Eq.~\eqref{eq:ns_attn_def},
which yields a nonnegative denominator by construction and improves numerical stability.
The induced affinity is $\widetilde{k}_\theta(x,y)=\widetilde{\phi}_\theta(x)^\top\widetilde{\phi}_\theta(y)$. Theorem~\ref{theo:Non_stationary_RFF} applies only to the signed features before ELU+1; it does not establish unbiasedness of the transformed affinity or the normalized attention ratio. Appendix~\ref{app:prf_gap} characterizes this transformed kernel and quantifies its difference from the signed-feature similarity.

\subsection{Flow-based Spectral Density Modeling}
\label{sec:flow_spectral}
\begin{figure*}[t]
    \centering
    \includegraphics[width=1.0\textwidth]{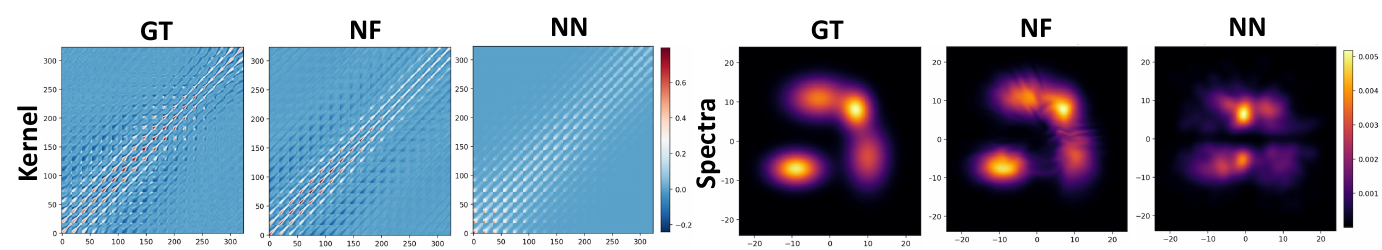}
    \caption{
    Visualization of kernels and spectral densities \(p_(\boldsymbol{\omega}_1,\boldsymbol{\omega}_2)\) under NF-based and NN-based implicit parameterizations.
    NF-based spectral modeling recovers more structured spectra and induces more faithful kernels.
    }
    \label{fig:method_spectra}
\end{figure*}

As established in Section~\ref{sec:nonstationary_kernel}, learning expressive nonstationary attention kernels can be reformulated as learning a bivariate spectral density
\(p(\boldsymbol{\omega}_1,\boldsymbol{\omega}_2)\).
Prior work such as GMM-RKS~\citep{pal2022learning} models the spectral distribution using Gaussian mixture models.
Although this improves expressiveness over fixed kernels, the flexibility and computational cost depend heavily on the number of mixture components, which may limit scalability in practice.
IKAN~\citep{song2021implicit} further generalizes kernelized attention through implicit neural parameterization.
However, the absence of an explicit spectral density hinders principled regularization, and may lead to spectral collapse, thereby degrading model performance.
As shown in Figure~\ref{fig:method_spectra}, NN-based implicit spectral parameterization may fail to recover the underlying spectral structure, leading to poorly learned kernels.

To address these limitations, we need to design a spectral modeling framework that is expressive and scalable while retaining an explicit and regularizable spectral representation.
We therefore introduce a flow-based approach for spectral density learning.
Normalizing flows (NFs) define flexible probability distributions by transforming a simple base distribution through a sequence of invertible and differentiable mappings.
Formally, we first sample a latent variable
\begin{equation}
\mathbf{z} \sim p_0(\mathbf{z}) = \mathcal{N}(\mathbf{0}, \mathbf{I}),
\qquad \mathbf{z}\in\mathbb{R}^{2d_x},
\end{equation}
and obtain the spectral pair through an invertible transformation
\begin{equation}
(\boldsymbol{\omega}_1,\boldsymbol{\omega}_2) = F_\theta(\mathbf{z}),
\qquad 
\boldsymbol{\omega}_1,\boldsymbol{\omega}_2 \in \mathbb{R}^{d_x}.
\end{equation}
where $F_{\theta}$ denotes a composition of $N_f$ invertible and differentiable transformations.
By the change-of-variables formula, the induced spectral density is given by
\begin{equation}
p_{\theta}(\boldsymbol{\omega}_1, \boldsymbol{\omega}_2)
=
p_0(\mathbf{z})
\left|
\det
\left(
\frac{\partial F_{\theta}(\mathbf{z})}{\partial \mathbf{z}}
\right)
\right|^{-1}.
\end{equation}
In this work, we parameterize $F_{\theta}$ with RealNVP transformations, which provide an efficient architecture with tractable Jacobian determinants \citep{dinh2016density}.
Therefore, sampling spectral pairs for random feature construction reduces to drawing $\mathbf{z}\sim p_0(\mathbf{z})$ and applying
$(\boldsymbol{\omega}_1,\boldsymbol{\omega}_2)=F_{\theta}(\mathbf{z})$.

To avoid overfitting the spectral density, we further introduce a spectral regularizer.
Specifically, we regularize the learned spectral density
\(p_{\theta}(\boldsymbol{\omega}_1,\boldsymbol{\omega}_2)\)
using the NG-SM spectral density as a reference prior
\(p_{\mathrm{NGSM}}(\boldsymbol{\omega}_1,\boldsymbol{\omega}_2)\),
which is defined as a bivariate Gaussian mixture~\citep{yang2025multi}.
The regularization term is
\begin{equation}
\mathcal{L}_{\mathrm{spec}}
=
\mathrm{KL}
\left(
p_{\theta}(\boldsymbol{\omega}_1,\boldsymbol{\omega}_2)
\,\|\, 
p_{\mathrm{NGSM}}(\boldsymbol{\omega}_1,\boldsymbol{\omega}_2)
\right),
\end{equation}
and is estimated by Monte Carlo samples from
\(p_{\theta}(\boldsymbol{\omega}_1,\boldsymbol{\omega}_2)\).
This term acts as a soft regularizer that stabilizes spectral learning while keeping the flow parameterization flexible.

Finally, with the flexible and stable flow-based spectral kernel attention (FSKA) introduced above, we further ask \emph{whether the expressiveness of the kernel can simplify the attention architecture itself}. We turn to this question next.

\subsection{Simplifying Attention via Kernel Expressiveness}
Standard attention uses learnable query, key, and value projections together with a fixed kernel function to model token interactions, resulting in a large number of projection parameters and reduced efficiency in long-sequence settings. We next simplify FSKA by removing the query and key projections and fixing the value projection.
Specifically, the learned kernel is expressive enough to model token similarity directly on the input representations, so explicit query and key projections can be removed. 
For the value branch, we use a simple fixed projection to provide a stable aggregation space, while leaving adaptive token interactions to the learned kernel.

\subsubsection{Removing Query and Key Projections}
In standard attention, input representations $\mathbf{X}$ are first projected into query and key spaces by learnable matrices, and the attention matrix is determined by a fixed similarity function evaluated on $(\mathbf{q}_i, \mathbf{k}_j)$.  In other words, token interactions are jointly determined by the learned projections and the predefined similarity function.

In contrast, our formulation models token interactions directly with a flexible kernel defined on the input space. 
When the kernel family is sufficiently expressive, the projected similarity can be represented directly as kernel acting on the original input representations $\mathbf{X}$.
This enables removing explicit query and key projections while preserving the expressiveness of kernel-based similarity modeling.
We formalize this observation below.

\begin{theorem}[Absorption of Query--Key Projections]
\label{thm:qk_absorption}
Let $\mathcal{X} \subset \mathbb{R}^{d_x}$ be a compact set. 
Let $\mathcal{K}$ be a family of functions 
$\tilde{k}: \mathcal{X} \times \mathcal{X} \rightarrow \mathbb{R}$ 
that is uniformly dense in $C(\mathcal{X} \times \mathcal{X})$.

Let $k \in C(\mathcal{U} \times \mathcal{V})$ be any continuous similarity function, 
where $\mathcal{U}$ and $\mathcal{V}$ are compact sets. 
Let $\mathbf{W}_Q \in \mathbb{R}^{d_x \times d_q}$ and 
$\mathbf{W}_K \in \mathbb{R}^{d_x \times d_k}$ be linear mappings such that 
$\mathcal{X} \mathbf{W}_Q \subset \mathcal{U}$ and 
$\mathcal{X} \mathbf{W}_K \subset \mathcal{V}$. 
Define
\begin{equation}
f(\mathbf{x}_1, \mathbf{x}_2) 
=
k(\mathbf{x}_1 \mathbf{W}_Q, \mathbf{x}_2 \mathbf{W}_K).
\end{equation}
Then, for any $\epsilon > 0$, there exists 
$\tilde{k}_{\epsilon} \in \mathcal{K}$ such that
\begin{equation}
\sup_{\mathbf{x}_1,\mathbf{x}_2 \in \mathcal{X}}
\left| 
f(\mathbf{x}_1, \mathbf{x}_2) 
- 
\tilde{k}_{\epsilon}(\mathbf{x}_1, \mathbf{x}_2) 
\right|
\le \epsilon .
\end{equation}
\end{theorem}

\begin{proof}
    See Appendix~\ref{app:proof_qk_absorption} for the full proof.
\end{proof}

Theorem~\ref{thm:qk_absorption} is conditional on its uniform-density assumption and does not establish that a finite implementation satisfies that assumption. Under this assumption, any similarity function induced by linear projections can be uniformly approximated by a kernel defined directly on the original input representations. This implies that explicit query and key projections are not strictly necessary for modeling pairwise similarities in expressive kernel attention, and introduce additional parameters.

Motivated by this observation, we remove query and key projections and compute attention directly on the input representations:
\begin{equation}
\mathrm{Attn}(\mathbf{X})_i 
= 
\frac{
\sum_{j=1}^{L} \widetilde{k}_\theta(\mathbf{x}_i, \mathbf{x}_j)\, \mathbf{v}_j
}{
\sum_{j'=1}^{L} \widetilde{k}_\theta(\mathbf{x}_i, \mathbf{x}_{j'})
}.
\label{eq:remove_qk_attention}
\end{equation}
This simplification of the attention mechanism reduces projection-related parameters by shifting token interaction modeling from learned projections to the expressive kernel attention mechanism.

\subsubsection{Value Representation with Fixed Projections}
\label{sec:value_representation}
After removing query and key projections, the attention weights are computed directly from the input representations through the learned kernel.
A natural question is whether the value branch in Eq.~\eqref{eq:remove_qk_attention} can also be simplified by directly using the original token representations as values.
However, directly using $\mathbf{X}$ provides only a single value representation space.
In a multi-head setting, this may limit representation diversity across heads and lead to redundant outputs, even if different heads apply distinct kernel attention. Therefore, the value branch cannot be removed in the same way as query and key projections.

The reason is that the value branch plays a different role from the query and key branches.
Query and key projections affect similarity computation and can therefore be absorbed by a sufficiently expressive learned kernel.
By contrast, the value branch defines the feature space in which token information is aggregated.
Directly using $\mathbf{X}$ forces all heads to aggregate in the same original representation space, whereas value projections provide different feature subspaces to capture complementary sequence information.
However, these value projections do not need to be learnable. Since the learned kernel already adaptively determines the aggregation weights over the values, the value projection mainly provides the feature space in which the selected token information is represented.

Motivated by this observation, we design a non-trainable value projection $\mathbf{W}_V \in \mathbb{R}^{d_x \times d_v}$:
\begin{equation}
\mathrm{Attn}(\mathbf{X})_i
=
\sum_{j=1}^{L}
\alpha_{ij}^{\theta} \mathbf{x}_j \mathbf{W}_V,
\qquad
\alpha_{ij}^{\theta}
=
\frac{
\widetilde{k}_{\theta}(\mathbf{x}_i,\mathbf{x}_j)
}{
\sum_{j'=1}^{L} \widetilde{k}_{\theta}(\mathbf{x}_i,\mathbf{x}_{j'})
}.
\end{equation}
Here, the positive-feature affinity \(\widetilde{k}_{\theta}\) determines the input-dependent similarity structure and assigns adaptive aggregation weights, while the fixed projection \(\mathbf{W}_V\) determines how token information is represented before aggregation.

Finally, we show that FSKA can be simplified with expressive kernel attention by removing the query and key projections and fixing the value projection, since the learned kernel already models token similarities directly while the value projection mainly provides the representation space for aggregation.
This simplified design reduces projection-related parameter size while preserving the expressive capacity of the attention mechanism. This formulation also motivates our multi-head design: different heads use head-specific learned kernels to induce diverse input-dependent similarity patterns, together with fixed orthogonal value projections that aggregate information in complementary feature subspaces.

\subsection{Multi-Head Kernel Attention}
\label{sec:multi_head_kernel_attention}
We next extend the proposed kernel attention to the multi-head setting. We assign each head a head-specific learned kernel together with a fixed orthogonal value projection, so that different heads can induce diverse input-dependent similarity patterns while aggregating information in complementary feature subspaces.

Formally, for the $h$-th head, the attention weights are defined as
\begin{equation}
\alpha_{ij}^{(h)}=
\frac{\widetilde{k}_{\theta_h}(\mathbf{x}_i,\mathbf{x}_j)}
{\sum_{j'=1}^{L}\widetilde{k}_{\theta_h}(\mathbf{x}_i,\mathbf{x}_{j'})},
\end{equation}
where each head's spectral density is parameterized by a normalizing flow as in Section~\ref{sec:flow_spectral}, and $\widetilde{k}_{\theta_h}$ denotes the affinity after the positive feature transformation in Section~\ref{sec:nonstationary_kernel}. The corresponding head output is
\begin{equation}
\mathrm{Attn}^{(h)}(\mathbf{X})_i
=\sum_{j=1}^{L}\alpha_{ij}^{(h)}\mathbf{x}_j\mathbf{W}_V^{(h)},
\end{equation}
where $\mathbf{W}_V^{(h)}\in\mathbb{R}^{d_x\times d_v}$ is a fixed value projection.

To preserve diversity in the value space without introducing learnable value projections, we design the value projections of different heads to be mutually orthogonal:
\begin{equation}
\left(\mathbf{W}_V^{(h)}\right)^\top\mathbf{W}_V^{(h')}=\mathbf{0},\qquad h\ne h'.
\end{equation}
This orthogonality encourages different heads to aggregate information in distinct value subspaces, thereby reducing redundancy among head outputs.

The final multi-head output is obtained by concatenating all head outputs:
\begin{equation}
\mathrm{MHAttn}(\mathbf{X})_i
=\mathrm{Concat}\big(\mathrm{Attn}^{(1)}(\mathbf{X})_i,\dots,\mathrm{Attn}^{(H)}(\mathbf{X})_i\big).
\end{equation}
This design preserves multi-head diversity through two mechanisms: head-specific kernel attention and orthogonal fixed value subspaces. Compared with standard multi-head attention, this design removes explicit query and key projections and avoids learnable value projections, thereby reducing projection-related parameters while preserving diverse token interactions through head-specific learned kernels and complementary value subspaces.

%% file: context/experiment.tex
\section{Experiments}
\label{sec:experiments}
In this section, we first examine the contribution of each design component in Section~\ref{sec:ablation}, then report sentence and full-document classification results in Sections~\ref{sec:text_classification} and~\ref{sec:long_seq_cls}, and finally analyze parameter efficiency and memory scaling in Section~\ref{sec:param_memory_scalability} to evaluate practical efficiency. 
Additional analyses, including PRF stability and similarity differences, nonstationary regression, hyperparameter sensitivity, and LRA-style benchmark settings, are provided in Appendix~\ref{app:additional_experiments}. The new document-classification protocol and supplementary spectral-learning, query--key, and regularization ablations are collected in Appendices~\ref{app:document_protocol} and~\ref{app:document_ablations}.

\subsection{Ablation Study}
\label{sec:ablation}
We start by conducting an ablation study on MNIST to isolate the effects of the main architectural choices in our attention module.
The results are reported in Table~\ref{table:ablation}.
Overall, the ablation supports two central claims of our method:
FSKA improves model performance by learning expressive nonstationary kernels for attention through a flow-based spectral density parameterization.
Moreover, the flexible learned kernel enables a simpler projection design while preserving the expressive capacity of the model.

First, our method achieves the best performance among the compared methods, demonstrating the benefit of explicit flow-based spectral kernel learning.
Unlike the standard Transformer, which relies on a fixed softmax kernel, and IKAN, whose kernel is learned implicitly without an explicit spectral density, our method directly parameterizes a bivariate spectral density with normalizing flows.
This yields a flexible and regularizable nonstationary kernel attention mechanism, improving the model's ability to capture complex token interactions.

Second, the ablation confirms that an expressive kernel enables a simpler projection design.
Removing projection branches from a standard Transformer causes a large performance drop, because the fixed softmax kernel alone is insufficient for adaptive similarity modeling.
In contrast, our model removes query--key projections while maintaining strong performance, as the learned kernel directly models token interactions in the input representation space.
Moreover, making the value projection learnable brings little additional benefit over the fixed-value variant.
These results support our claim that query--key projections become less necessary, and fixed value projections are sufficient, once adaptive routing is governed by a sufficiently expressive kernel.

In addition, PRFs are crucial for stable linearized attention when the number of random features is limited, e.g., \(M=96\).
Direct RFF features severely degrade performance because their signed components may cancel in the attention denominator and destabilize normalization.
Our positive feature map prevents this issue by ensuring nonnegative features.

\input{table/ablation}

\subsection{Sentence Classification}
\label{sec:text_classification}
We further evaluate the proposed FSKA on sentence-level classification tasks.
We conduct experiments on six benchmark datasets, including DBPedia, AGNews, TREC, RT-Polarity, SST-2, and SST-5, covering topic classification, question classification, and sentiment classification.\footnote{Details of the datasets and baselines are provided in Appendix~\ref{app:sentence_classification_details}.}
We compare with the standard Transformer, MIKAN, Performer, RKA, GMM-RKS, KPCA, and MetaLA~\citep{vaswani2017attention,song2021implicit,choromanski2020rethinking,peng2021random,pal2022learning,teo2024unveiling,chou2024metala}, and report F1-macro over five runs in Table~\ref{tab:text_cls_f1}.\footnote{Additional accuracy results are reported in Appendix~\ref{app:sentence_classification_details}, Table~\ref{tab:text_cls_acc}.}
Overall, FSKA achieves the strongest or highly competitive performance across the six datasets.
Notably, FSKA improves the Transformer from $76.07$ to $78.74$ on SST-2 and from $35.11$ to $38.16$ on SST-5.

These results demonstrate the benefit of explicit flow-based spectral kernel learning.
Unlike the standard Transformer and fixed-form attention variants such as Performer, RKA, KPCA, and MetaLA, whose similarity structures are largely determined by predefined attention forms, FSKA directly learns an explicit bivariate spectral density with NFs.
This yields a flexible and regularizable nonstationary kernel, allowing the model to better adapt token interactions to different classification tasks.

Compared with learnable kernel methods such as GMM-RKS and MIKAN, FSKA also shows more consistent performance.
GMM-RKS relies on a Gaussian-mixture spectral parameterization, which may suffer from increasingly difficult optimization as the number of mixture components grows.
MIKAN learns kernelized attention implicitly, but lacks an explicit and regularizable spectral density representation.
In contrast, FSKA combines explicit spectral modeling with the expressiveness of normalizing flows, leading to stronger overall performance across both Accuracy and F1-macro.

\input{table/text_seq}

\subsection{Document Classification with Full Sequences}
\label{sec:long_seq_cls}

We extend the evaluation to SCOTUS, BBC News, 20 Newsgroups (20NG), and arXiv.
SCOTUS opinions contain approximately 5--9k tokens and arXiv papers average approximately 6k tokens;
BBC News and 20NG provide complementary medium-length documents, averaging approximately 380 and 300 words, respectively.
Each tokenized document is provided in full to every model, without truncation, chunking, or sliding-window segmentation.
In addition to the attention baselines, we include FlashAttention~\citep{dao2022flashattention} and Mamba~\citep{gu2023mamba}.
Evaluation details and the available FSKA settings are given in Appendix~\ref{app:document_protocol}.
Table~\ref{tab:long_seq_cls} reports accuracy and macro-averaged F1, as mean $\pm$ standard deviation over three independent runs.

FSKA obtains the highest mean macro-F1 on all four datasets and the highest mean accuracy on BBC News, 20NG, and arXiv.
On SCOTUS, Mamba achieves the highest accuracy (59.84\%), while FSKA attains a higher macro-F1 (39.68\% versus 37.91\%).
The macro-F1 margins over the strongest competing baseline are 1.77, 1.13, 0.39, and 1.57 percentage points on SCOTUS, BBC News, 20NG, and arXiv, respectively.
These results show competitive document classification performance across the evaluated sequence lengths; the reported mean differences alone do not establish statistical significance.

\input{table/long_seq}

\input{context/efficiency}

\section{Conclusion}
\label{sec:conclusion}
To address the limited flexibility of fixed or stationary attention kernels, we proposed FSKA, a flow-based spectral kernel attention framework for learning expressive nonstationary attention kernels.
FSKA parameterizes the spectral density with normalizing flows and regularizes it with a structured spectral prior, enabling flexible kernel learning while remaining compatible with RFF-based linear attention.
We further showed that expressive kernel learning can simplify attention by removing query--key projections and using fixed orthogonal value projections. Experiments show that FSKA achieves competitive or improved performance on sentence and full-document classification tasks.
It also reduces parameter sizes by approximately 40\%--47\% compared with most baselines and maintains favorable linear-attention memory scaling on long-sequence benchmarks.
Overall, these results highlight FSKA as an expressive and parameter-efficient attention mechanism.

%% file: table/ablation.tex
\begin{table*}[t]
\centering
\caption{
Ablation study on architectural design choices for attention mechanisms on MNIST classification.
We report test accuracy (Acc).
\cmark indicates that the corresponding component is used, \xmark indicates that it is removed, and -- indicates that the component is not applicable.
}
\label{table:ablation}
\resizebox{1.0\linewidth}{!}{
\begin{tabular}{lcccc|c}
\toprule
Model 
& QK Proj. 
& Learnable V 
& Flexible Kernel
& Positive RF 
& Acc $\uparrow$ \\
\midrule

Transformer 
& \cmark 
& \cmark 
& \xmark 
& -- 
& 0.8556 \\

Transformer w/o QKV
& \xmark 
& \xmark 
& \xmark
& -- 
& 0.5267 \\

IKAN 
& \cmark 
& \cmark 
& \cmark 
& \xmark 
& 0.8656 \\

\midrule

\textbf{Ours (FSKA)} 
& \xmark 
& \xmark 
& \cmark 
& \cmark 
& \textbf{0.8722} \\

+ Direct RFF 
& \xmark 
& \xmark 
& \cmark 
& \xmark 
& 0.3956 \\

+ Learnable V 
& \xmark 
& \cmark 
& \cmark 
& \cmark 
& 0.8725 \\

\bottomrule
\end{tabular}
}
\end{table*}

%% file: table/text_seq.tex
\begin{table*}[t!]
\caption{F1-macro (\%) on text classification benchmark datasets.
Mean and standard deviation are computed over five runs.
The best result per column is bolded and the second-best is underlined.
}
\label{tab:text_cls_f1}

\centering

\setlength{\tabcolsep}{1.6mm}{
\resizebox{\linewidth}{!}{
\begin{tabular}{lc||cccccc}
\toprule
\multirow{1}{*}{MODEL} 
& \multirow{1}{*}{Venue}
& DBPedia & AGNews & TREC & RT-Polarity & SST-2 & SST-5 \\
\cmidrule(lr){3-8}

\multirow{1}{*}{METRIC} 
& 
& \multicolumn{6}{c}{F1-macro ($\uparrow$)} \\
\midrule \midrule

Transformer  
& NeurIPS 2017
& 98.13 $\pm$ 0.15 
& 90.43 $\pm$ 0.07
& \underline{86.14 $\pm$ 0.93}
& 70.52 $\pm$ 0.45
& 76.07 $\pm$ 1.10
& 35.11 $\pm$ 1.46 \\

MIKAN  
& AAAI 2021
& 97.74 $\pm$ 0.01 
& \underline{91.15 $\pm$ 0.24}
& 85.47 $\pm$ 1.07
& 71.30 $\pm$ 0.92
& 76.04 $\pm$ 0.86
& 32.84 $\pm$ 0.39 \\

Performer  
& ICLR 2021
& 98.12 $\pm$ 0.09 
& 90.23 $\pm$ 0.31
& 85.70 $\pm$ 0.57
& 70.88 $\pm$ 0.62
& 77.24 $\pm$ 0.41
& 35.92 $\pm$ 0.96 \\

RKA  
& ICLR 2021
& \underline{98.26 $\pm$ 0.01} 
& 90.80 $\pm$ 0.28
& 85.56 $\pm$ 0.75
& 71.15 $\pm$ 0.72
& 77.08 $\pm$ 1.46
& 36.10 $\pm$ 0.49 \\

GMM-RKS  
& TMLR 2022
& 96.76 $\pm$ 0.21 
& 90.58 $\pm$ 0.41
& 85.36 $\pm$ 0.91
& 71.05 $\pm$ 1.02
& 76.81 $\pm$ 0.89
& 36.42 $\pm$ 0.84 \\

KPCA 
& NeurIPS 2024
& 97.59 $\pm$ 0.14 
& 90.20 $\pm$ 0.49
& 85.90 $\pm$ 0.96
& \textbf{72.16 $\pm$ 1.58}
& \underline{77.62 $\pm$ 1.13}
& \underline{36.97 $\pm$ 1.36} \\

MetaLA
& NeurIPS 2024
& 97.24 $\pm$ 0.15 
& 91.14 $\pm$ 0.36
& 85.07 $\pm$ 1.02
& 70.28 $\pm$ 1.71
& 77.03 $\pm$ 0.69
& 35.27 $\pm$ 1.41 \\

\rowcolor[HTML]{e6f3ff}
Ours
& --
& \textbf{98.33 $\pm$ 0.13}
& \textbf{91.20 $\pm$ 0.37}
& \textbf{86.21 $\pm$ 1.02}
& \underline{72.08 $\pm$ 1.44}
& \textbf{78.74 $\pm$ 1.15}
& \textbf{38.16 $\pm$ 1.40} \\

\bottomrule
\end{tabular}
}
}

\end{table*}

%% file: table/long_seq.tex
\begin{table*}[t]
\centering
\caption{Full-document classification results (\%). Values are mean $\pm$ standard deviation over three independent runs. Bold denotes the highest mean within each column. F1 denotes macro-F1. Documents are processed in full; BBC News and 20NG are shorter than SCOTUS and arXiv.}
\label{tab:long_seq_cls}
\scriptsize
\setlength{\tabcolsep}{2pt}
\begin{tabular*}{\linewidth}{@{\extracolsep{\fill}}l*{8}{c}@{}}
\toprule
\multirow{2}{*}{Model} & \multicolumn{2}{c}{SCOTUS} & \multicolumn{2}{c}{BBC News} & \multicolumn{2}{c}{20NG} & \multicolumn{2}{c}{arXiv} \\
\cmidrule(lr){2-3}\cmidrule(lr){4-5}\cmidrule(lr){6-7}\cmidrule(lr){8-9}
& Acc. & F1 & Acc. & F1 & Acc. & F1 & Acc. & F1 \\
\midrule
Transformer & 55.12 $\!\pm\!$ 1.24 & 35.76 $\!\pm\!$ 0.91 & 88.43 $\!\pm\!$ 1.52 & 88.01 $\!\pm\!$ 1.57 & 24.91 $\!\pm\!$ 0.73 & 24.18 $\!\pm\!$ 0.69 & 75.62 $\!\pm\!$ 1.61 & 74.90 $\!\pm\!$ 1.66 \\
FlashAttention & 55.68 $\!\pm\!$ 0.79 & 36.21 $\!\pm\!$ 0.87 & 88.91 $\!\pm\!$ 0.49 & 88.47 $\!\pm\!$ 0.54 & 25.46 $\!\pm\!$ 0.70 & 24.92 $\!\pm\!$ 0.67 & 76.18 $\!\pm\!$ 0.58 & 75.41 $\!\pm\!$ 0.63 \\
Performer & 52.38 $\!\pm\!$ 0.77 & 31.05 $\!\pm\!$ 0.88 & 83.74 $\!\pm\!$ 1.26 & 82.96 $\!\pm\!$ 1.21 & 25.33 $\!\pm\!$ 0.81 & 24.97 $\!\pm\!$ 0.78 & 69.81 $\!\pm\!$ 1.04 & 68.93 $\!\pm\!$ 1.09 \\
Mamba & \textbf{59.84 $\!\pm\!$ 0.68} & 37.91 $\!\pm\!$ 0.78 & 89.34 $\!\pm\!$ 0.46 & 88.96 $\!\pm\!$ 0.51 & 28.91 $\!\pm\!$ 0.63 & 28.47 $\!\pm\!$ 0.66 & 76.92 $\!\pm\!$ 0.54 & 76.24 $\!\pm\!$ 0.59 \\
RKA & 53.06 $\!\pm\!$ 0.99 & 31.72 $\!\pm\!$ 1.26 & 85.18 $\!\pm\!$ 0.48 & 84.71 $\!\pm\!$ 0.55 & 28.62 $\!\pm\!$ 0.64 & 28.34 $\!\pm\!$ 0.67 & 70.15 $\!\pm\!$ 1.58 & 69.42 $\!\pm\!$ 1.62 \\
GMM-RKS & 39.44 $\!\pm\!$ 0.93 & 12.31 $\!\pm\!$ 0.82 & 66.92 $\!\pm\!$ 0.75 & 65.81 $\!\pm\!$ 0.79 & 27.13 $\!\pm\!$ 0.44 & 23.62 $\!\pm\!$ 0.31 & 73.08 $\!\pm\!$ 0.51 & 72.47 $\!\pm\!$ 0.28 \\
KPCA-Scaled & 56.37 $\!\pm\!$ 0.71 & 37.42 $\!\pm\!$ 0.83 & 89.56 $\!\pm\!$ 0.43 & 89.18 $\!\pm\!$ 0.49 & 22.74 $\!\pm\!$ 0.68 & 22.01 $\!\pm\!$ 0.72 & 76.84 $\!\pm\!$ 0.55 & 76.09 $\!\pm\!$ 0.60 \\
MIKAN & 40.28 $\!\pm\!$ 0.86 & 17.05 $\!\pm\!$ 0.74 & 84.31 $\!\pm\!$ 0.59 & 83.92 $\!\pm\!$ 0.63 & 18.46 $\!\pm\!$ 0.71 & 17.83 $\!\pm\!$ 0.68 & 64.92 $\!\pm\!$ 0.70 & 64.15 $\!\pm\!$ 0.73 \\
\rowcolor[HTML]{e6f3ff}
FSKA (Ours) & 58.91 $\!\pm\!$ 1.12 & \textbf{39.68 $\!\pm\!$ 0.71} & \textbf{90.76 $\!\pm\!$ 0.39} & \textbf{90.31 $\!\pm\!$ 0.44} & \textbf{30.24 $\!\pm\!$ 0.57} & \textbf{28.86 $\!\pm\!$ 0.61} & \textbf{78.43 $\!\pm\!$ 0.48} & \textbf{77.81 $\!\pm\!$ 0.53} \\
\bottomrule
\end{tabular*}
\end{table*}

%% file: context/efficiency.tex
\subsection{Parameter Efficiency and Memory Scalability}
\label{sec:param_memory_scalability}

We further evaluate the practical efficiency of our design by analyzing trainable parameter allocation and training memory growth under an LRA-style benchmark~\citep{tay2020long}. \footnote{Additional details are provided in Appendix~\ref{app:lra-additional}.}
All methods use the same backbone architecture and are measured under identical training conditions and hardware environment, with details provided in Appendix~\ref{app:sentence_classification_details}.

As shown in Fig.~\ref{fig:lra_param_mem_tradeoff}, FSKA achieves a favorable efficiency--scalability profile.
In Fig.~\ref{fig:lra_param_breakdown}, FSKA uses roughly \(40\%\)--\(47\%\) fewer trainable parameters than most baseline attention modules.
This reduction mainly comes from eliminating the learnable query and key projections and replacing the learnable value projection with a fixed mapping, while retaining adaptive similarity modeling through the learned nonstationary kernel.

In terms of memory, FSKA follows a low-growth trend close to linear-attention baselines.
As shown in Fig.~\ref{fig:lra_mem_scaling}, when the sequence length increases to \(4096\), FSKA keeps training memory below \(10\)K MB, whereas quadratic-style attention mechanisms grow to around \(80\)K--\(120\)K MB.
This favorable scaling comes from the RFF-based kernel attention design, which avoids explicitly forming the full quadratic attention matrix.

\begin{figure*}[t]
    \centering
    \begin{minipage}[t]{0.49\textwidth}
        \centering
        \includegraphics[width=\linewidth]{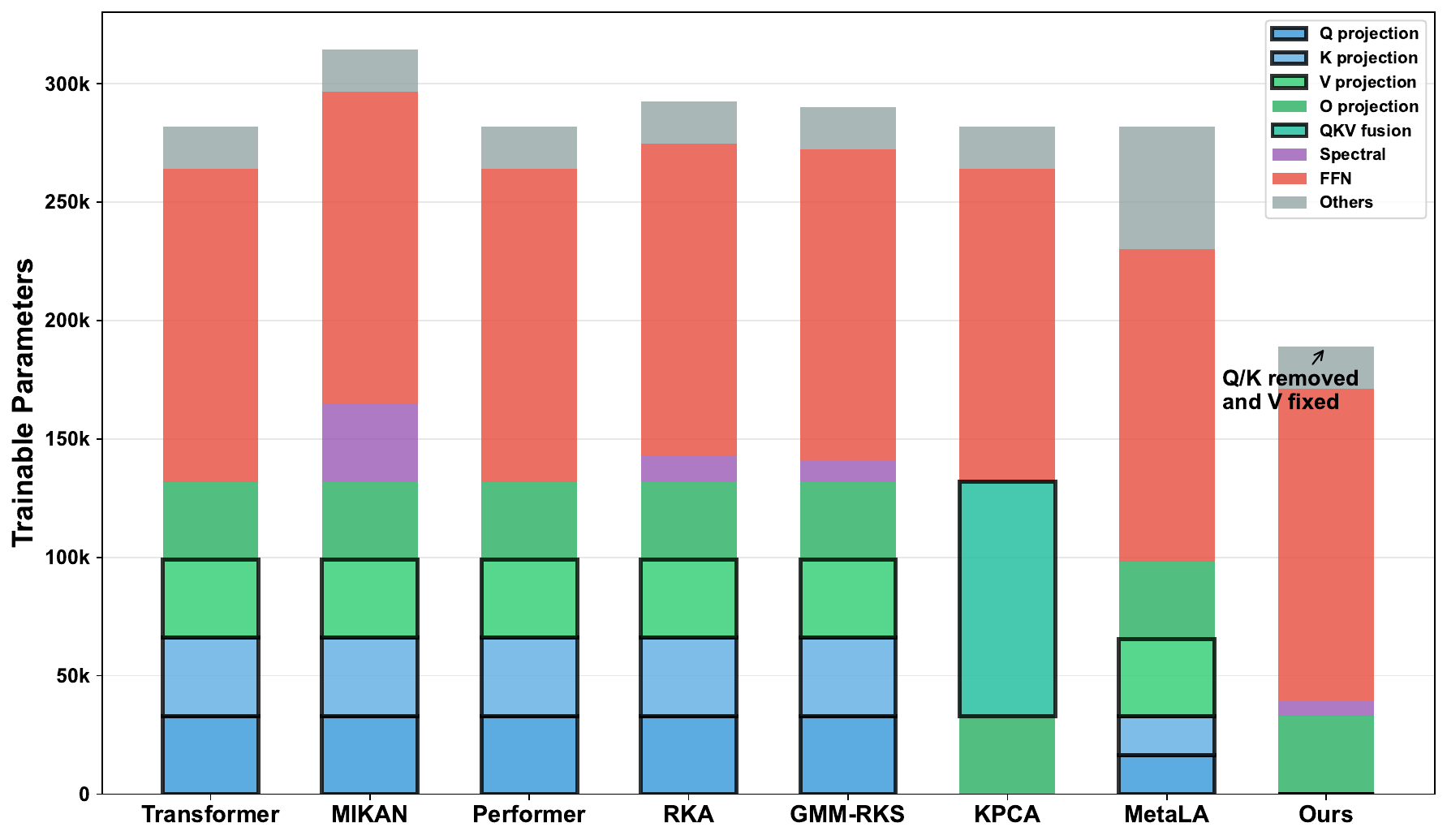}
        
        \subcaption{Trainable parameter breakdown by component.}
        \label{fig:lra_param_breakdown}
    \end{minipage}
    \hfill
    \begin{minipage}[t]{0.49\textwidth}
        \centering
        \includegraphics[width=\linewidth]{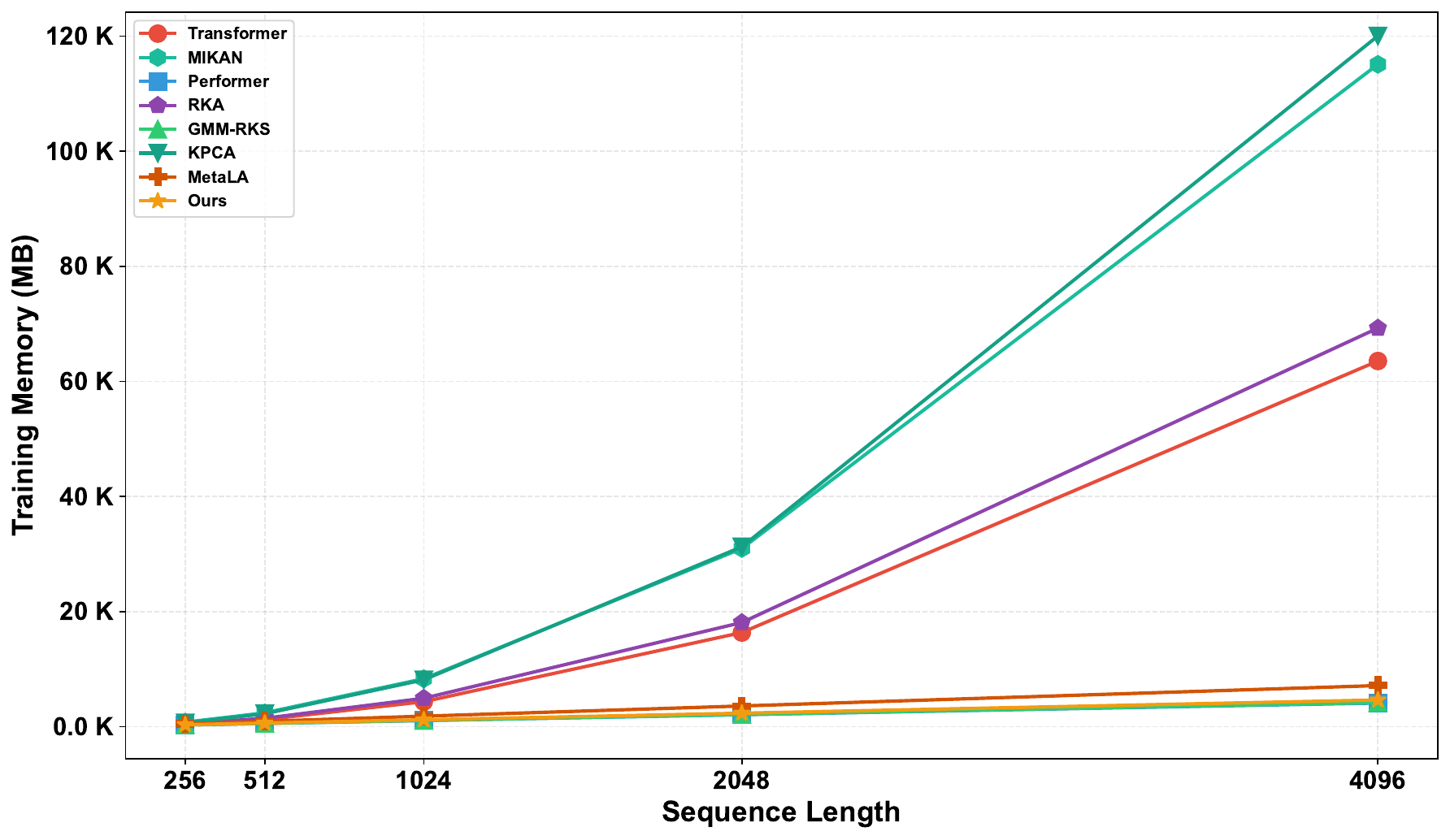}
        
        \subcaption{Training memory versus sequence length.}
        \label{fig:lra_mem_scaling}
    \end{minipage}
    \caption{
    Parameter efficiency and memory scalability under the LRA-style benchmark.
    Left: different methods allocate trainable parameters differently across projection, spectral, and feed-forward components.
    Right: training memory growth with increasing sequence length.
    }
    \label{fig:lra_param_mem_tradeoff}
\end{figure*}

%% file: context/app.tex
\section{Additional Proofs}
\label{app:proofs}

\subsection{Proof of Theorem~\ref{theo:Non_stationary_RFF}}
\label{app:Non_stationary_RFF}

\begin{proof}
By Theorem~\ref{theo:Uni_Bochner}, the nonstationary kernel admits the
spectral representation
\begin{equation}
\label{eq:app_ubo_exp}
\begin{aligned}
k(\mathbf{x}_1,\mathbf{x}_2)
=
\frac{1}{4}
\int
\Big(
&
\exp\!\left(i(\mathbf{x}_1 \boldsymbol{\omega}_1^\top
-\mathbf{x}_2 \boldsymbol{\omega}_2^\top)\right)
+
\exp\!\left(i(\mathbf{x}_1 \boldsymbol{\omega}_2^\top
-\mathbf{x}_2 \boldsymbol{\omega}_1^\top)\right)
\\
&+
\exp\!\left(i(\mathbf{x}_1 \boldsymbol{\omega}_1^\top
-\mathbf{x}_2 \boldsymbol{\omega}_1^\top)\right)
+
\exp\!\left(i(\mathbf{x}_1 \boldsymbol{\omega}_2^\top
-\mathbf{x}_2 \boldsymbol{\omega}_2^\top)\right)
\Big)
\,u(\mathrm{d}\boldsymbol{\omega}_1,\mathrm{d}\boldsymbol{\omega}_2).
\end{aligned}
\end{equation}
For real-valued kernels, we take the real part of Eq.~\eqref{eq:app_ubo_exp}.
When \(u\) admits a density \(p(\boldsymbol{\omega}_1,\boldsymbol{\omega}_2)\),
this gives
\begin{equation}
\label{eq:app_ubo_cos}
\begin{aligned}
k(\mathbf{x}_1,\mathbf{x}_2)
=
\frac{1}{4}
\mathbb{E}_{p(\boldsymbol{\omega}_1,\boldsymbol{\omega}_2)}
\Big[
&
\cos(\mathbf{x}_1 \boldsymbol{\omega}_1^\top
-\mathbf{x}_2 \boldsymbol{\omega}_2^\top)
+
\cos(\mathbf{x}_1 \boldsymbol{\omega}_2^\top
-\mathbf{x}_2 \boldsymbol{\omega}_1^\top)
\\
&+
\cos(\mathbf{x}_1 \boldsymbol{\omega}_1^\top
-\mathbf{x}_2 \boldsymbol{\omega}_1^\top)
+
\cos(\mathbf{x}_1 \boldsymbol{\omega}_2^\top
-\mathbf{x}_2 \boldsymbol{\omega}_2^\top)
\Big].
\end{aligned}
\end{equation}

Let
\(\{(\boldsymbol{\omega}_1^{(m)},\boldsymbol{\omega}_2^{(m)})\}_{m=1}^{M/2}\)
be i.i.d. sampled row vectors from \(p(\boldsymbol{\omega}_1,\boldsymbol{\omega}_2)\).
The corresponding Monte Carlo estimator is
\begin{equation}
\label{eq:app_mc_estimator}
\begin{aligned}
\widehat{k}(\mathbf{x}_1,\mathbf{x}_2)
=
\frac{1}{2M}
\sum_{m=1}^{M/2}
\Big[
&
\cos(\mathbf{x}_1 \boldsymbol{\omega}_1^{(m)\top}
-\mathbf{x}_2 \boldsymbol{\omega}_2^{(m)\top})
+
\cos(\mathbf{x}_1 \boldsymbol{\omega}_2^{(m)\top}
-\mathbf{x}_2 \boldsymbol{\omega}_1^{(m)\top})
\\
&+
\cos(\mathbf{x}_1 \boldsymbol{\omega}_1^{(m)\top}
-\mathbf{x}_2 \boldsymbol{\omega}_1^{(m)\top})
+
\cos(\mathbf{x}_1 \boldsymbol{\omega}_2^{(m)\top}
-\mathbf{x}_2 \boldsymbol{\omega}_2^{(m)\top})
\Big].
\end{aligned}
\end{equation}
Using the trigonometric identity
\(\cos(a-b)=\cos(a)\cos(b)+\sin(a)\sin(b)\), we can rewrite
Eq.~\eqref{eq:app_mc_estimator} as
\begin{equation}
\label{eq:app_feature_expansion}
\begin{aligned}
\widehat{k}(\mathbf{x}_1,\mathbf{x}_2)
=
\frac{1}{2M}
\sum_{m=1}^{M/2}
\Big[
&
\Big(
\cos(\mathbf{x}_1 \boldsymbol{\omega}_1^{(m)\top})
+
\cos(\mathbf{x}_1 \boldsymbol{\omega}_2^{(m)\top})
\Big)
\\[-1mm]
&\qquad\times\Big(
\cos(\mathbf{x}_2 \boldsymbol{\omega}_1^{(m)\top})
+
\cos(\mathbf{x}_2 \boldsymbol{\omega}_2^{(m)\top})
\Big)
\\
&+
\Big(
\sin(\mathbf{x}_1 \boldsymbol{\omega}_1^{(m)\top})
+
\sin(\mathbf{x}_1 \boldsymbol{\omega}_2^{(m)\top})
\Big)
\\[-1mm]
&\qquad\times\Big(
\sin(\mathbf{x}_2 \boldsymbol{\omega}_1^{(m)\top})
+
\sin(\mathbf{x}_2 \boldsymbol{\omega}_2^{(m)\top})
\Big)
\Big].
\end{aligned}
\end{equation}
Now define the randomized feature map
\begin{equation}
\label{eq:app_rff_feature_map}
\boldsymbol{\phi}(\mathbf{x})
=
\sqrt{\frac{1}{2M}}
\begin{bmatrix}
\cos\!\left(\mathbf{x} \boldsymbol{\omega}_1^{(1:M/2)\top}\right)
+
\cos\!\left(\mathbf{x} \boldsymbol{\omega}_2^{(1:M/2)\top}\right)
\\[1mm]
\sin\!\left(\mathbf{x} \boldsymbol{\omega}_1^{(1:M/2)\top}\right)
+
\sin\!\left(\mathbf{x} \boldsymbol{\omega}_2^{(1:M/2)\top}\right)
\end{bmatrix}
\in \mathbb{R}^{M}.
\end{equation}
Then Eq.~\eqref{eq:app_feature_expansion} can be written compactly as
\begin{equation}
\widehat{k}(\mathbf{x}_1,\mathbf{x}_2)
=
\boldsymbol{\phi}(\mathbf{x}_1)^\top \boldsymbol{\phi}(\mathbf{x}_2) .
\end{equation}
Finally, since the spectral pairs are sampled independently from
\(p(\boldsymbol{\omega}_1,\boldsymbol{\omega}_2)\), the Monte Carlo estimator is
unbiased:
\begin{equation}
\mathbb{E}
\left[
\boldsymbol{\phi}(\mathbf{x}_1)^\top \boldsymbol{\phi}(\mathbf{x}_2)
\right]
=
k(\mathbf{x}_1,\mathbf{x}_2).
\end{equation}
Therefore, the kernel \(k(\mathbf{x}_1,\mathbf{x}_2)\) admits the unbiased
random-feature estimator
\(\boldsymbol{\phi}(\mathbf{x}_1)^\top \boldsymbol{\phi}(\mathbf{x}_2)\).
\end{proof}

\subsection{Proof of Theorem~\ref{thm:qk_absorption}}
\label{app:proof_qk_absorption}

\begin{proof}[Proof]
Since $\mathcal{X}$ is compact and the linear mappings defined by $\mathbf{W}_Q$ and $\mathbf{W}_K$ are continuous, their images under right-multiplication
\[
\mathcal{X} \mathbf{W}_Q
=
\{\mathbf{x} \mathbf{W}_Q: \mathbf{x}\in\mathcal{X}\},
\qquad
\mathcal{X} \mathbf{W}_K
=
\{\mathbf{x} \mathbf{W}_K: \mathbf{x}\in\mathcal{X}\}
\]
are compact subsets of $\mathcal{U}$ and $\mathcal{V}$, respectively.

Now define the map
\[
T:\mathcal{X}\times \mathcal{X} \rightarrow \mathcal{U}\times \mathcal{V},
\qquad
T(\mathbf{x}_1,\mathbf{x}_2)
=
(\mathbf{x}_1\mathbf{W}_Q,\mathbf{x}_2\mathbf{W}_K).
\]
Because the mappings $\mathbf{x} \mapsto \mathbf{x}\mathbf{W}_Q$ and $\mathbf{x} \mapsto \mathbf{x}\mathbf{W}_K$ are continuous, the product map $T$ is also continuous. 
Since $k\in C(\mathcal{U}\times\mathcal{V})$, the composition
\[
f(\mathbf{x}_1,\mathbf{x}_2)
=
k\big(\mathbf{x}_1\mathbf{W}_Q,\mathbf{x}_2\mathbf{W}_K\big)
=
(k\circ T)(\mathbf{x}_1,\mathbf{x}_2)
\]
is continuous on $\mathcal{X}\times\mathcal{X}$. 
Hence,
\[
f\in C(\mathcal{X}\times\mathcal{X}).
\]

By the assumption that $\mathcal{K}$ is uniformly dense in $C(\mathcal{X}\times\mathcal{X})$, for any $\epsilon>0$ and any function 
$f\in C(\mathcal{X}\times\mathcal{X})$, there exists 
$\tilde{k}_{\epsilon}\in\mathcal{K}$ such that
\[
\|f-\tilde{k}_{\epsilon}\|_{\infty}
=
\sup_{\mathbf{x}_1,\mathbf{x}_2\in\mathcal{X}}
\left|
f(\mathbf{x}_1,\mathbf{x}_2)
-
\tilde{k}_{\epsilon}(\mathbf{x}_1,\mathbf{x}_2)
\right|
\le \epsilon .
\]
Substituting the definition of $f$ yields
\[
\sup_{\mathbf{x}_1,\mathbf{x}_2\in\mathcal{X}}
\left|
k(\mathbf{x}_1\mathbf{W}_Q,\mathbf{x}_2\mathbf{W}_K)
-
\tilde{k}_{\epsilon}(\mathbf{x}_1,\mathbf{x}_2)
\right|
\le \epsilon .
\]
This proves that the projected similarity function induced by query and key projections can be uniformly approximated by a kernel defined directly on the original input space. 
\end{proof}

\section{Experimental Details and Additional Experiments}
\label{app:additional_experiments}

All experiments were conducted on a cloud server equipped with 2 NVIDIA Tesla V100 GPUs (16GB memory each), a 10-core Intel Xeon Platinum 81xx series CPU, 64GB RAM, and 200GB storage.

\subsection{Additional Analysis on Positive Random Features}
\label{app:prf_stability}

We provide an additional analysis of the denominator stability issue in linearized kernel attention.
For a query token \(q_i\), the denominator of random-feature attention is given by
\begin{equation}
D_i
=
\boldsymbol{\phi}(q_i)^\top \sum_{j=1}^{L}\boldsymbol{\phi}(k_j),
\end{equation}
where \(\boldsymbol{\phi}(\cdot)\) denotes the random feature map.
Since the sine and cosine components in the direct random Fourier feature construction can take both positive and negative values, finite-sample cancellation may occur in \(D_i\).
This can lead to very small or even unstable normalization values, especially when the number of random features \(M\) is limited.

To study this effect, we simulate the denominator distribution under different random feature dimensions
\[
M \in \{32,64,128,256,512,1024,2048,4096\}.
\]
For each value of \(M\), we compute the linear-attention denominator \(D_i\) using direct random Fourier features and collect its empirical distribution over tokens and random trials.
The original plotting script measures the fraction of negative or relatively small denominators:
\begin{equation}
\widehat{P}\!\left(D_i<0\ \text{or}\ |D_i|<10^{-3}\operatorname{median}(|D|)\right).
\end{equation}
The median is computed over batch and token denominators separately for each feature count and method. This statistic differs from an absolute near-zero threshold. The diagnostic uses 32 batches of 64 queries and keys in dimension 16, independently drawn from $\mathcal{N}(0,0.25I)$, with Gaussian RBF frequencies $\omega\sim\mathcal{N}(0,I)$. It does not train a flow or a Transformer. Here the plotted $M$ counts frequency pairs, yielding $2M$ feature coordinates.
We compare direct RFF features with the positive random feature (PRF) transformation used in our model,
\begin{equation}
\widetilde{\boldsymbol{\phi}}(x) = g(\boldsymbol{\phi}(x)),
\qquad
g(u)=
\begin{cases}
u+1, & u>0,\\
\exp(u), & u\le 0,
\end{cases}
\end{equation}
which corresponds to \(\mathrm{ELU}(u)+1\).
This transformation ensures \(\widetilde{\boldsymbol{\phi}}(x)\geq 0\) element-wise.

\begin{figure*}[t]
    \centering
    \begin{subfigure}[t]{1.0\textwidth}
        \centering
        \includegraphics[width=\linewidth]{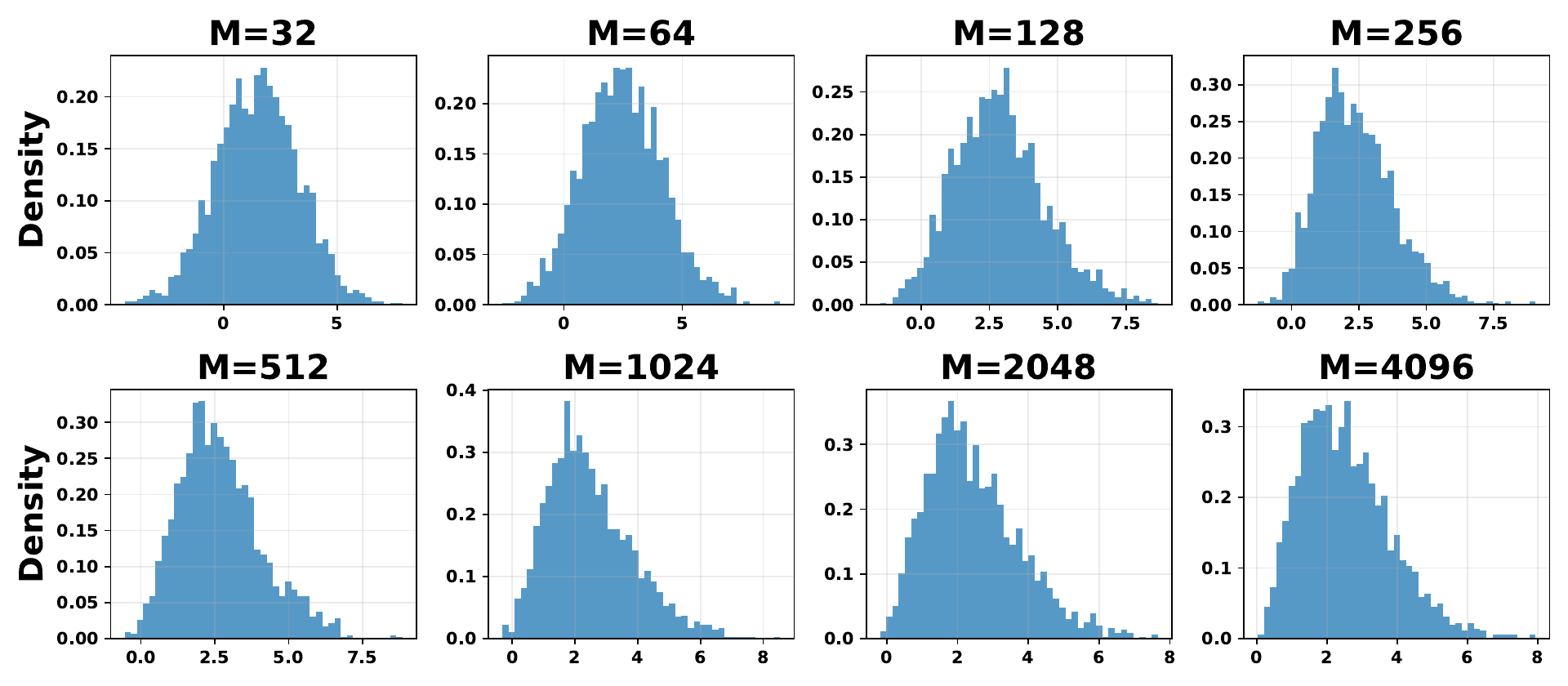}
        \caption{Distribution of the linear-attention denominator of RFFs under different random feature dimensions \(M\).}
        \label{fig:app_denom_distribution}
    \end{subfigure}
    \hfill
    \begin{subfigure}[t]{1.0\textwidth}
        \centering
        \includegraphics[width=.72\linewidth]{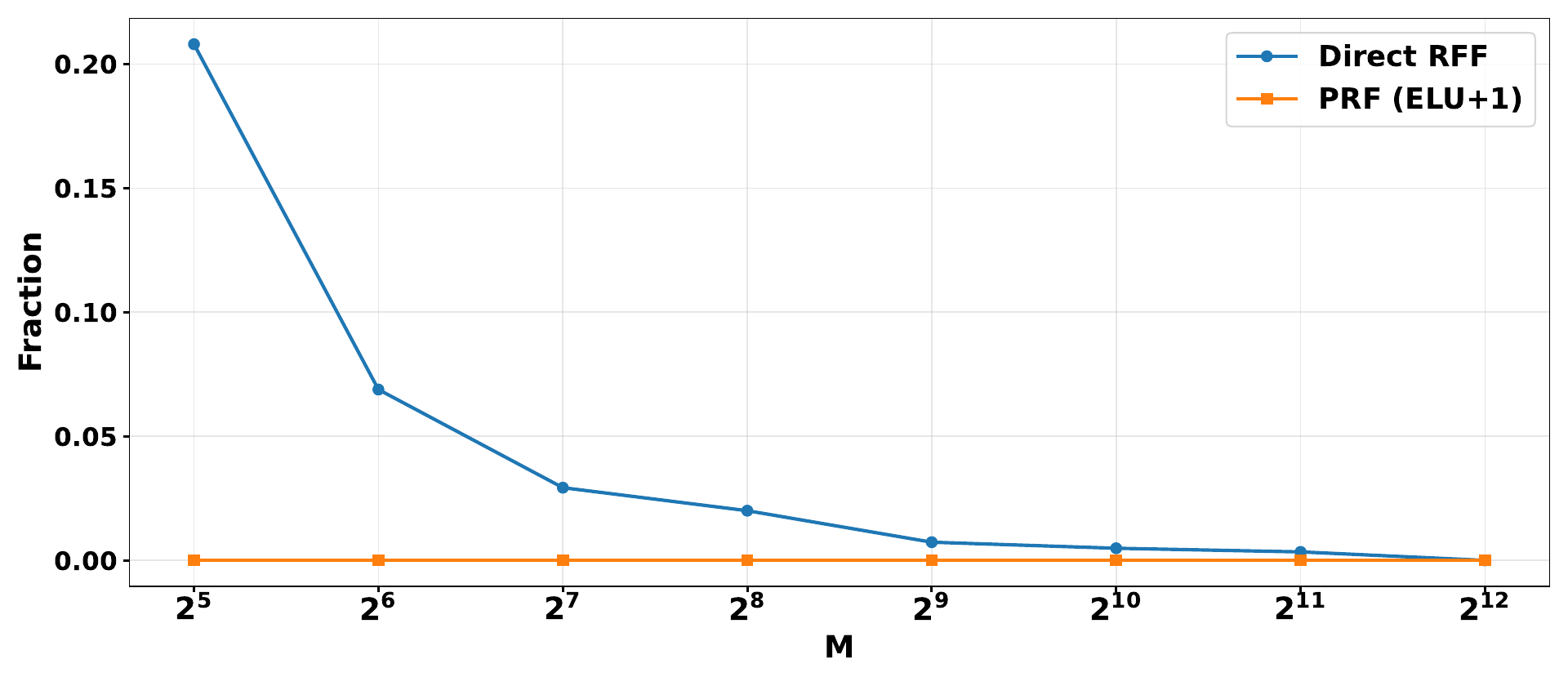}
        \caption{Fraction of negative or relatively small denominators versus \(M\).}
        \label{fig:app_unstable_frac_vs_M}
    \end{subfigure}
    \caption{
    Additional stability analysis of RFFs and PRFs.
    }
    \label{fig:app_prf_stability}
\end{figure*}

Figure~\ref{fig:app_denom_distribution} shows the empirical denominator distributions for different values of \(M\).
When \(M\) is small, the denominator distribution of RFFs has non-negligible mass near zero, indicating that the finite random-feature approximation can produce unstable normalization.
As \(M\) increases, the distribution becomes more concentrated away from zero, suggesting that using more random features reduces the probability of cancellation.

Figure~\ref{fig:app_unstable_frac_vs_M} further quantifies this trend by reporting the fraction of negative or relatively small denominators.
For direct RFF features, the unstable fraction decreases overall as \(M\) increases, but it is not completely eliminated under finite random-feature dimensions.
In contrast, PRF keeps this fraction at zero across all tested values of \(M\).

These results motivate the use of PRF in our linear-attention implementation.
Although direct RFFs provide an unbiased approximation to the target kernel, their signed feature components can make the denominator numerically unstable.
Using a positive feature transformation sacrifices the signed direct feature representation but gives a more stable attention normalization in finite-sample settings.
This is particularly important in practice because \(M\) cannot be made arbitrarily large due to memory and runtime constraints.
Therefore, PRF provides a simple and effective stabilization mechanism for scalable FSKA attention.

\input{context/prf_gap}

\subsection{Toy Example for Nonstationary Regression}
\label{sec:toy_nonstationary}

We design a controlled toy benchmark to evaluate whether different attention kernels can model
position-dependent nonstationarity under distribution shift.
The focus is on simultaneously capturing global trend consistency and local regime transitions.

\paragraph{Data generation.}
We construct a 1D dataset on a normalized grid
\[
t_i=\frac{i-1}{N-1},\quad i=1,\dots,N,\quad N=500,\quad t_i\in[0,1].
\]
The target is
\[
y_i = f(t_i) + \varepsilon_i,\qquad \varepsilon_i\sim\mathcal N\!\left(0,\sigma^2(t_i)\right),
\]
where \(f(\cdot)\) is a smooth position-dependent mixture of regime-specific oscillatory components:
\[
f(t)=\sum_{k=1}^{K} w_k(t)\,s_k(t)+m(t),\qquad K=4.
\]
The mixing weights are defined by logistic gates over absolute position:
\[
s_1(t)=\sigma(\alpha(t-\tau_1)),\;
s_2(t)=\sigma(\alpha(t-\tau_2)),\;
s_3(t)=\sigma(\alpha(t-\tau_3)),
\]
\[
w_0(t)=1-s_1(t),\;
w_1(t)=s_1(t)-s_2(t),\;
w_2(t)=s_2(t)-s_3(t),\;
w_3(t)=s_3(t),
\]
with \(\sigma(u)=1/(1+e^{-u})\), \(\alpha>0\), and \(\tau_1<\tau_2<\tau_3\).
Each \(s_k(t)\) uses different sinusoidal frequency/phase patterns, \(m(t)\) is a position-dependent mean-shift/trend term, and \(\sigma^2(t)\) introduces heteroscedastic noise.
Inputs include time, sinusoidal/polynomial terms, auxiliary variables, and gate features.
We use a 70/30 train-test split with a fixed split.

We compare 9 models: Transformer, IKAN-direct, MIKAN, five text-kernel baselines
(KPCAScaled, Performer, RKA, GMM-RKS, MetaLA), and ours FSKA.
All models are trained with Adam under the same split and evaluated by test MSE.
Figure~\ref{fig:toy_nonstationary_regression} reports per-model prediction plots and MSE.

\begin{figure*}[t]
    \centering
    \includegraphics[width=\textwidth]{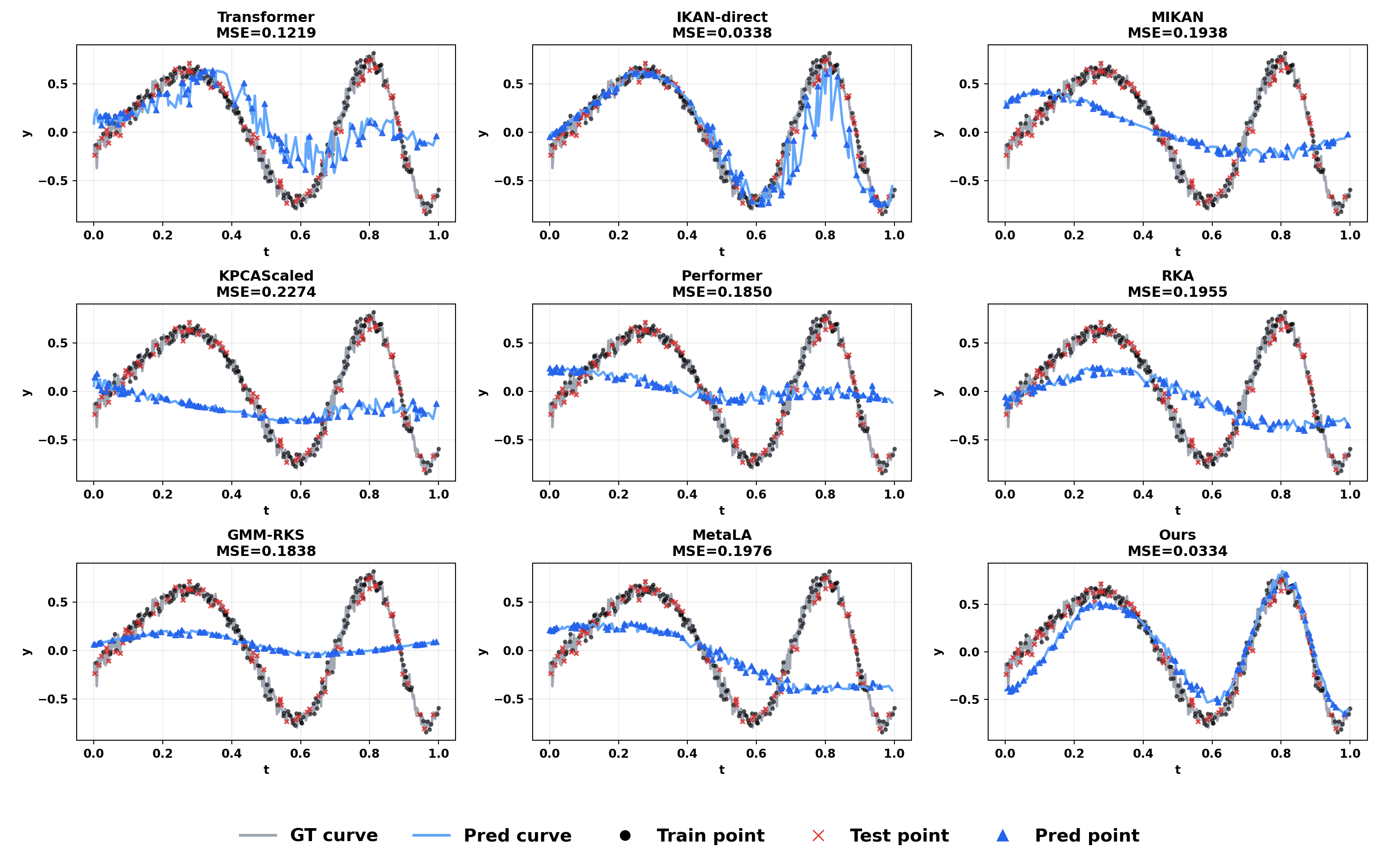}
    \caption{Regression plots on the toy nonstationary example across baselines.}
    \label{fig:toy_nonstationary_regression}
\end{figure*}

Ours FSKA achieves the best test error (MSE \(=0.0334\)), slightly outperforming IKAN-direct (0.0338),
and substantially improving over Transformer (0.1219), MIKAN (0.1938), and most text-kernel baselines (\(\approx 0.18\!-\!0.23\)).
Qualitatively, Ours FSKA tracks regime boundaries and high-curvature regions more faithfully while preserving global phase alignment.
These observations are consistent with our hypothesis:
an expressive NF-parameterized spectral kernel improves adaptation to nonstationary structure,
while simplified projection design avoids unnecessary optimization entanglement.

\subsection{Hyperparameter Sensitivity}
\label{subsec:hyperparameter_sensitivity}

We investigate the sensitivity of our model to three key hyperparameters: the NGSM mixture count $Q$, the RFF dimension $M$, and the normalizing flow depth $N_f$. All experiments are conducted on MNIST, and results are summarized in Table~\ref{tab:param_sensitivity}.

\paragraph{NGSM mixture count $Q$.}
The mixture count $Q$ controls the expressiveness of the NGSM kernel approximation. Performance improves steadily as $Q$ increases from 1 to 10 (0.8483 $\to$ 0.8743), confirming that a richer mixture provides a more flexible kernel representation. However, increasing $Q$ beyond 10 yields no further gains and even slight degradation ($Q{=}20$: 0.8431), suggesting overfitting of the mixture components. Notably, runtime remains nearly constant across all values of $Q$ (75--77s), indicating that the mixture computation is efficiently parallelized and does not constitute a bottleneck.

\paragraph{RFF dimension $M$.}
The RFF dimension $M$ governs the rank of the random feature approximation to the implicit kernel. Increasing $M$ from 32 to 96 yields a modest accuracy improvement (0.8656 $\to$ 0.8743) with only a mild runtime increase, demonstrating favorable cost-accuracy scaling. Further increasing to $M{=}128$ pushes accuracy to 0.8822 but nearly doubles the runtime (149.20s), as the quadratic cost of the feature map begins to dominate. The default $M{=}96$ offers a balanced trade-off between accuracy (0.8743) and efficiency (85.32s).

\paragraph{Normalizing flow depth $N_f$.}
The flow depth $N_f$ determines the capacity of the normalizing flow to transform the base distribution into a complex kernel shape. Increasing $N_f$ from 2 to 10 produces a consistent accuracy gain (0.8503 $\to$ 0.8707), but at a substantial runtime cost that scales approximately linearly ($\sim$13s per additional layer). Beyond $N_f{=}4$, the marginal accuracy improvement diminishes while the computational overhead continues to grow. The default $N_f{=}5$ strikes an effective balance, achieving near-peak accuracy without excessive overhead.

Overall, The default configuration ($Q{=}10$, $M{=}96$, $N_f{=}5$) consistently achieves strong performance.

\input{table/hyparameter}

\subsection{Additional Details for Sentence Classification}
\label{app:sentence_classification_details}

\subsubsection{Additional Results for Sentence Classification}

Table~\ref{tab:text_cls_acc} reports additional accuracy results on six sentence classification benchmarks.
Overall, FSKA achieves competitive performance across all datasets and obtains the best accuracy on DBPedia, TREC, RT-Polarity, and SST-2, while ranking second on AGNews.
Notably, FSKA improves over the standard Transformer from $70.93$ to $73.18$ on RT-Polarity, and from $75.99$ to $77.58$ on SST-2.
These gains suggest that the proposed flow-based spectral kernel can provide more effective token interactions than standard dot-product attention on sentiment classification tasks.

\begin{table*}[t!]
\caption{Accuracy (\%) on text classification benchmark datasets.
Mean and standard deviation are computed over five runs.
The best result per column is bolded and the second-best is underlined.}
\label{tab:text_cls_acc}
\centering

\setlength{\tabcolsep}{1.6mm}{
\scalebox{0.75}{
\begin{tabular}{lc||cccccc}
\toprule
\multirow{1}{*}{MODEL} 
& \multirow{1}{*}{Venue}
& DBPedia & AGNews & TREC & RT-Polarity & SST-2 & SST-5 \\
\cmidrule(lr){3-8}

\multirow{1}{*}{METRIC} 
& 
& \multicolumn{6}{c}{Accuracy ($\uparrow$)} \\
\midrule \midrule

Transformer  
& NeurIPS 2017
& 98.13 $\pm$ 0.15 
& 90.45 $\pm$ 0.07
& 86.67 $\pm$ 0.90
& 70.93 $\pm$ 0.83
& 75.99 $\pm$ 1.06
& 38.16 $\pm$ 0.67 \\

MIKAN  
& AAAI 2021
& 97.94 $\pm$ 0.01 
& 91.15 $\pm$ 0.24
& 86.60 $\pm$ 0.16
& 71.57 $\pm$ 0.38
& 75.40 $\pm$ 0.40
& 37.33 $\pm$ 0.73 \\

Performer  
& ICLR 2021
& 98.12 $\pm$ 0.09 
& 90.27 $\pm$ 0.29
& 85.00 $\pm$ 0.86
& 70.31 $\pm$ 0.42
& \underline{77.26 $\pm$ 0.31}
& \underline{39.44 $\pm$ 0.42} \\

RKA  
& ICLR 2021
& \underline{98.16 $\pm$ 0.01} 
& 90.82 $\pm$ 0.28
& \underline{86.80 $\pm$ 0.82}
& 72.65 $\pm$ 0.82
& 76.77 $\pm$ 0.88
& \textbf{39.87 $\pm$ 0.50} \\

GMM-RKS  
& TMLR 2022
& 96.82 $\pm$ 0.18 
& 90.61 $\pm$ 0.34
& 85.74 $\pm$ 0.88
& 71.42 $\pm$ 0.96
& 76.38 $\pm$ 0.73
& 38.72 $\pm$ 0.61 \\

KPCA 
& NeurIPS 2024
& 97.60 $\pm$ 0.10
& \textbf{91.22 $\pm$ 0.39}
& 85.82 $\pm$ 0.72
& \underline{72.94 $\pm$ 1.36}
& 76.72 $\pm$ 0.91
& 39.15 $\pm$ 1.12 \\

MetaLA
& NeurIPS 2024
& 97.62 $\pm$ 0.11 
& 91.15 $\pm$ 0.47
& 85.31 $\pm$ 0.83
& 72.52 $\pm$ 1.56
& 74.84 $\pm$ 0.78
& 37.31 $\pm$ 1.24 \\

\rowcolor[HTML]{e6f3ff}
Ours   
& --
& \textbf{98.23 $\pm$ 0.17}
& \underline{91.20 $\pm$ 0.31}
& \textbf{86.87 $\pm$ 0.76}
& \textbf{73.18 $\pm$ 1.22}
& \textbf{77.58 $\pm$ 0.95}
& 39.01 $\pm$ 1.21 \\

\bottomrule
\end{tabular}
}
}
\end{table*}

\subsubsection{Baselines and Their Parameter Settings}

All models are trained within the same sentence classification benchmark pipeline. Unless otherwise noted, the shared backbone hyperparameters are:
embedding dimension $128$, hidden dimension $128$, $4$ attention heads, $2$ encoder layers, feed-forward dimension $256$, dropout $0.1$, maximum sequence length $256$, batch size $32$, learning rate $10^{-3}$, weight decay $10^{-4}$, cosine learning-rate scheduling, $5$ warm-up epochs, early stopping patience $10$, and at most $50$ training epochs. Text is lowercased, tokenized with a whitespace tokenizer, and represented with a vocabulary built from the training split only.

\paragraph{Transformer.}
The Transformer baseline is the standard encoder-style multi-head self-attention model of Vaswani et al.~\citep{vaswani2017attention}. In our implementation, each layer contains learned $Q$, $K$, $V$, and output projections, with hidden dimension $128$, $4$ heads, and per-head dimension $32$. We use $2$ Transformer encoder layers, feed-forward dimension $256$, and dropout $0.1$. This model serves as the quadratic-attention reference.

\paragraph{MIKAN.}
MIKAN is our implicit kernel attention baseline, implemented in the benchmark as a multi-head implicit kernel attention network~\citep{song2021implicit}. It uses the same backbone dimensions as the Transformer baseline, but replaces softmax attention with the implicit kernel attention module. The attention layer uses $M=64$ random Fourier features per head.

\paragraph{Performer.}
Performer is the FAVOR+ linear-attention baseline of Choromanski et al.~\citep{choromanski2020rethinking}. It uses the same encoder backbone as the other models, but replaces quadratic softmax attention with positive orthogonal random feature attention. In our code, the number of Performer random features is not explicitly overridden by the benchmark driver, so the implementation falls back to its internal default
$n_{\text{features}}=\lfloor d_{\text{head}}\log d_{\text{head}} \rfloor$.
With $d_{\text{head}}=32$, this yields $n_{\text{features}}=\lfloor 32 \log 32 \rfloor = 110$ random features.
The projection matrix is sampled as a Gaussian orthogonal random matrix and is redrawn every $1000$ training calls.

\paragraph{RKA.}
RKA is a random kernel attention baseline with learned random features~\citep{peng2021random}. It shares the same encoder backbone dimensions as the other baselines, but generates Fourier features through a small generator network rather than using fixed random projections. In our implementation, RKA uses $64$ random features, i.e., \texttt{n\_features=64}. Since no explicit \texttt{noise\_dims} are passed by the benchmark driver, the code uses the default generator dimensions
$[d_{\text{head}}, 2d_{\text{head}}, d_{\text{head}}] = [32, 64, 32]$.
The generator uses leaky-ReLU hidden activations and a tanh output activation.

\paragraph{GMM-RKS.}
GMM-RKS is a Gaussian-mixture random kitchen sinks baseline that learns the spectral distribution through per-head Gaussian mixture parameters~\citep{pal2022learning}. It uses the same shared encoder backbone hyperparameters as the other baselines. In our implementation, the feature map uses \texttt{n\_dims=64}, with one learnable mean vector and one learnable full covariance matrix per attention head. 

\paragraph{FlashAttention (document classification).}
FlashAttention is an IO-aware algorithm for exact softmax attention that uses tiling to reduce data transfers between GPU memory levels~\citep{dao2022flashattention}. We include it in the full-document classification comparison in Section~\ref{sec:long_seq_cls} as an efficient exact-attention baseline.

\paragraph{Mamba (document classification).}
Mamba is a sequence model based on selective state spaces, whose input-dependent parameters control the propagation of information along the sequence~\citep{gu2023mamba}. It scales linearly with sequence length and provides a non-attention comparison in the full-document classification experiments in Section~\ref{sec:long_seq_cls}.

\subsubsection{Datasets, Sources, and Access Links}

We evaluate on six standard sentence classification benchmarks spanning sentiment analysis, question classification, news topic classification, ontology classification, and review rating prediction.

\paragraph{RT Polarity.}
RT Polarity is a binary sentiment classification benchmark built from movie review snippets by Pang and Lee~\citep{pang2005seeing}. The public source is the Cornell Movie Review Data page:
\href{https://www.cs.cornell.edu/people/pabo/movie-review-data/}{https://www.cs.cornell.edu/people/pabo/movie-review-data/}.
In our pipeline, RT Polarity is split into train/test using an $80/20$ stratified split, and then $10\%$ of the training portion is further held out for validation.

\paragraph{SST-2.}
SST-2 is the binary sentiment subset of the Stanford Sentiment Treebank, introduced by Socher et al.~\citep{socher2013recursive}. We load it from Stanford Sentiment Treebank release:
\href{https://www-nlp.stanford.edu/sentiment/treebank.html}{https://www-nlp.stanford.edu/sentiment/treebank.html}.
In our preprocessing code, SST-2 is reconstructed from the original SST sentence-level annotations by mapping sentiment scores to binary labels and discarding neutral examples. We use the official train/dev/test split provided with the SST release.

\paragraph{SST-5.}
SST-5 is the five-way fine-grained version of the Stanford Sentiment Treebank~\citep{socher2013recursive}. It is loaded from the same original Stanford Sentiment Treebank source as SST-2:
\href{https://www-nlp.stanford.edu/sentiment/treebank.html}{https://www-nlp.stanford.edu/sentiment/treebank.html}.
Unlike SST-2, it retains all five sentiment levels (very negative, negative, neutral, positive, very positive). We use the official train/dev/test split from the original release.

\paragraph{TREC.}
TREC is a six-class question classification dataset commonly associated with the setup of Li and Roth~\citep{li2002learning}. In our benchmark, it is loaded through the HuggingFace \texttt{datasets} package
\href{https://huggingface.co/docs/datasets}{https://huggingface.co/docs/datasets}
using the dataset identifier \texttt{trec}.

\paragraph{AG News.}
AG News is a four-class news topic classification dataset included in the large-scale character-level text classification benchmark of Zhang et al.~\citep{zhang2015character}. In our code, it is loaded through the HuggingFace \texttt{datasets} package
\href{https://huggingface.co/docs/datasets}{https://huggingface.co/docs/datasets}
using the dataset identifier \texttt{ag\_news}.
The benchmark uses the standard train/test split provided by the package and additionally reserves $10\%$ of the training set for validation via a stratified split.

\paragraph{DBpedia.}
DBpedia ontology classification is a 14-class topic classification benchmark also popularized by Zhang et al.~\citep{zhang2015character}. In our code, it is loaded through the HuggingFace \texttt{datasets} package
\href{https://huggingface.co/docs/datasets}{https://huggingface.co/docs/datasets}
using the dataset identifier \texttt{dbpedia\_14}.
As with AG News, we use the standard train/test split returned by the package and construct a stratified validation split from the training data.

\paragraph{Preprocessing and vocabulary construction.}
All datasets are lowercased and tokenized with a simple whitespace tokenizer. No punctuation or number stripping is enabled in the default benchmark configuration. A vocabulary is built only from the training split, using \texttt{<PAD>} and \texttt{<UNK>} as special tokens, with maximum vocabulary size $50{,}000$ and minimum token frequency $1$. All sequences are truncated or padded to a maximum length of $256$ tokens.

\subsection{Additional Details of LRA Experiments}
\label{app:lra-additional}

\paragraph{LRA-style synthetic benchmark.}
In addition to the sentence classification datasets above, we also conduct a Long Range Arena (LRA)-style scalability benchmark to evaluate memory usage and throughput under increasing context lengths. We construct synthetic input sequences with varying sequence lengths and measure training/inference memory together with throughput under a controlled architecture configuration. Concretely, for each model, we generate random input tensors $x \in \mathbb{R}^{B \times n \times d}$ with entries sampled from a standard Gaussian distribution, where $B$ is the batch size, $n$ is the sequence length, and $d$ is the model dimension. For training-time memory measurements, we also generate random targets of the same shape and optimize a mean-squared-error objective for several warm-up and measurement steps. We vary the sequence length over $\{256, 512, 1024, 2048, 4096\}$ while keeping the remaining model hyperparameters fixed. This setup is intended to isolate computational scaling behavior rather than downstream task accuracy, and therefore should be interpreted as an LRA-style systems benchmark rather than an evaluation on the official LRA datasets.

\input{context/document_experiments}

%% file: context/prf_gap.tex
\subsubsection{Kernel Induced by the Positive Transformation}
For a fixed collection of spectral samples $\Omega$, write the implemented signed map as $\phi_{\mathrm{impl},\theta}(x;\Omega)$ and the positive map as
\[
z_\theta(x;\Omega)=g(\phi_{\mathrm{impl},\theta}(x;\Omega)),\qquad g(u)=\operatorname{ELU}(u)+1.
\]
The finite-feature kernel actually used in attention is
\begin{equation}
K_{\mathrm{ELU},\theta}^{(m)}(x,y;\Omega)
=z_\theta(x;\Omega)^\top z_\theta(y;\Omega),
\label{eq:prf_induced_kernel}
\end{equation}
where $m$ counts spectral frequency pairs. This is the same affinity denoted by $\widetilde{k}_\theta$ in the main text. For arbitrary inputs $x_1,\ldots,x_n$ and coefficients $c_1,\ldots,c_n$,
\[
\sum_{i,j}c_i c_j K_{\mathrm{ELU},\theta}^{(m)}(x_i,x_j;\Omega)
=\left\|\sum_i c_i z_\theta(x_i;\Omega)\right\|_2^2\geq0.
\]
Thus, it is positive semidefinite for every fixed shared spectral draw. Its population counterpart is
\begin{equation}
\overline{K}_{\mathrm{ELU},\theta}^{(m)}(x,y)
=\mathbb{E}_{\Omega\sim p_\theta^{\otimes m}}
\left[K_{\mathrm{ELU},\theta}^{(m)}(x,y;\Omega)\right].
\label{eq:prf_population_kernel}
\end{equation}
The expectation is also positive semidefinite. For a general RealNVP spectral distribution, we do not have a simpler closed-form expression for this expectation. The superscript emphasizes that the feature normalization and therefore this transformed kernel depend on the finite feature count.

Theorem~\ref{theo:Non_stationary_RFF} proves unbiasedness only for the signed trigonometric feature map in Eq.~\eqref{eq:rff_def}, before any positive transformation and under its stated spectral distribution. It does not imply that $K_{\mathrm{ELU},\theta}^{(m)}$ estimates the same spectral kernel without bias. Nor does it assert unbiasedness of the normalized attention weights. ELU+1 changes the similarity function as well as its numerical behavior; its effect is not equivalent to merely adding a small constant to the denominator.

\subsubsection{Similarity Gap After ELU+1}
\label{app:prf_gap}
To quantify how the positive transformation changes similarities, we reproduce the diagnostic reported in the author response. It compares the signed Gram matrix $S_{ij}=\phi(x_i)^\top\phi(x_j)$ with $S^+_{ij}=g(\phi(x_i))^\top g(\phi(x_j))$. The associated analysis uses synthetic segmented tokens with varying mean directions and within-segment drift, with 96 tokens, dimension 16, and four segments. This is a feature-map diagnostic without end-to-end training: the reference script samples two Gaussian frequency sets, normalizes tokens and frequency vectors to unit length, and uses the $2\pi$ phase convention and the bivariate trigonometric construction before ELU+1. These diagnostic normalizations belong to the synthetic feature-map analysis, which is separate from the trained text-classification experiments.

Both matrices are cosine-normalized:
\[
C_{ij}=\frac{S_{ij}}{\sqrt{S_{ii}S_{jj}}},\qquad
C^+_{ij}=\frac{S^+_{ij}}{\sqrt{S^+_{ii}S^+_{jj}}}.
\]
The reported metrics are the Pearson correlation of their off-diagonal entries and the mean absolute discrepancy,
\[
\rho=\operatorname{Corr}\big((C_{ij})_{i\ne j},(C^+_{ij})_{i\ne j}\big),\qquad
\Delta=\frac{1}{L(L-1)}\sum_{i\ne j}|C_{ij}-C^+_{ij}|.
\]
Table~\ref{tab:prf_gap} retains the six numerical pairs from the author response. The signed matrix is a finite-feature reference, not an exact evaluation of the analytical spectral integral. Consequently, $\Delta$ measures disagreement between the two normalized feature similarities, not a direct error against the exact kernel $k$.

\begin{table}[htbp]
\centering
\caption{ELU+1 similarity gap from the author response. Metrics use off-diagonal token pairs after cosine normalization. The frequency-pair count $m$ is denoted by $M$ in the response and reference analysis.}
\label{tab:prf_gap}
\begin{tabular}{rrr}
\toprule
Frequency pairs $m$ & Correlation $\rho$ & Absolute discrepancy $\Delta$\\
\midrule
16 & 0.921 & 0.794\\
32 & 0.945 & 0.794\\
64 & 0.963 & 0.800\\
128 & 0.978 & 0.798\\
256 & 0.989 & 0.806\\
512 & 0.994 & 0.802\\
\bottomrule
\end{tabular}
\end{table}

Correlation increases from 0.921 to 0.994, indicating greater linear alignment between the relative similarity patterns in this diagnostic. However, the absolute discrepancy remains approximately 0.8. High correlation therefore does not mean equality of similarities, preservation of every pairwise ordering, or convergence to the original spectral kernel. PRF introduces a positive baseline and can compress similarity contrast; the persistent discrepancy is a limitation of interpreting the transformed affinity as the original signed spectral kernel. Its use in FSKA is an explicit modeling and normalization choice, supported here by denominator diagnostics and accompanied by this measured change in similarities.

%% file: table/hyparameter.tex
\begin{table}[t]
\centering
\caption{Parameter sensitivity analysis on MNIST.}
\label{tab:param_sensitivity}
\begin{tabular}{lcc}
\toprule
\textbf{Model} & \textbf{Acc} & \textbf{Runtime (s)} \\
\midrule
\textbf{Ours (FSKA: $Q{=}10$, $M{=}96$, $N_f{=}5$)} & 0.8743 & 85.32 \\
\midrule
\multicolumn{3}{l}{\textit{(a) NGSM Mixture Count $Q$}} \\
\quad $Q=1$   & 0.8483 & 76.52 \\
\quad $Q=3$   & 0.8543 & 76.30 \\
\quad $Q=5$   & 0.8612 & 76.39 \\
\quad $Q=20$  & 0.8431 & 75.32 \\
\quad $Q=40$  & 0.8643 & 76.82 \\
\midrule
\multicolumn{3}{l}{\textit{(b) RFF Dimension $M$}} \\
\quad $M=32$  & 0.8656 & 69.43 \\
\quad $M=64$  & 0.8692 & 81.70 \\
\quad $M=128$  & 0.8822 & 149.20 \\
\midrule
\multicolumn{3}{l}{\textit{(c) Normalizing Flow Depth $N_f$}} \\
\quad $N_f=2$  & 0.8503 & 64.52 \\
\quad $N_f=4$  & 0.8643 & 90.78 \\
\quad $N_f=6$  & 0.8623 & 119.53 \\
\quad $N_f=8$  & 0.8660 & 146.84 \\
\quad $N_f=10$ & 0.8707 & 173.78 \\
\bottomrule
\end{tabular}
\end{table}

%% file: context/document_experiments.tex
\subsection{Document Classification Protocol}
\label{app:document_protocol}

\paragraph{Tasks and sequence handling.}
The document classification experiments in Section~\ref{sec:long_seq_cls} cover legal opinions (SCOTUS), news articles (BBC News), newsgroup posts (20NG), and scientific papers (arXiv).
SCOTUS documents contain approximately 5--9k tokens and arXiv papers average approximately 6k tokens.
BBC News and 20NG contain shorter documents, averaging approximately 380 and 300 words, respectively; these word counts are distinct from tokenized sequence lengths.
Every model receives the full tokenized document, without truncation, chunking, or sliding-window segmentation.
The 256-token limit in Appendix~\ref{app:sentence_classification_details} applies only to the sentence classification benchmarks.

\paragraph{Data preparation and splits.}
SCOTUS is loaded from the \texttt{scotus} subset of LexGLUE, BBC News from \texttt{SetFit/bbc-news}, and arXiv from \texttt{ccdv/arxiv-classification}.\footnote{Dataset sources: \href{https://huggingface.co/datasets/coastalcph/lex_glue}{LexGLUE}, \href{https://huggingface.co/datasets/SetFit/bbc-news}{BBC News}, and \href{https://huggingface.co/datasets/ccdv/arxiv-classification}{arXiv classification}.}
20NG is loaded with \texttt{sklearn.datasets.fetch\_20newsgroups}, using the original train/test partitions and removing headers, footers, and quoted text.\footnote{\url{https://scikit-learn.org/stable/modules/generated/sklearn.datasets.fetch_20newsgroups.html}}
The preparation pipeline caps the training pool at 4,000 documents and the test set at 2,000 documents; smaller partitions are retained.
The loader reserves 10\% of the resulting training pool for a stratified validation split, using a fixed partition across model runs.
Table~\ref{tab:document_splits} summarizes the split sizes obtained from the dataset metadata and this loading procedure.

\begin{table}[htbp]
\centering
\caption{Document classification splits produced by the data preparation pipeline. Training counts exclude the held-out validation subset.}
\label{tab:document_splits}
\small
\begin{tabular}{lrrrr}
\toprule
Dataset & Classes & Train & Validation & Test \\
\midrule
SCOTUS & 13 & 3,600 & 400 & 1,400 \\
BBC News & 5 & 1,102 & 123 & 1,000 \\
20NG & 20 & 3,600 & 400 & 2,000 \\
arXiv & 11 & 3,600 & 400 & 2,000 \\
\bottomrule
\end{tabular}
\end{table}

\paragraph{Tokenization and vocabulary.}
The shared text preprocessor lowercases text, normalizes whitespace, and tokenizes by whitespace splitting.
Punctuation and numbers are retained.
The vocabulary is built on the training split only, with minimum token frequency 2 and a maximum of 50,000 entries including \texttt{<PAD>} and \texttt{<UNK>}.
Unseen tokens map to \texttt{<UNK>}; padding is accompanied by an attention mask.

\paragraph{FSKA implementation and training.}
The supplied FSKA classification configuration uses embedding and hidden dimensions of 128, four attention heads, two encoder layers, a feed-forward dimension of 256, and dropout 0.1.
The training implementation uses Adam with learning rate $10^{-3}$, weight decay $10^{-4}$, batch size 32, gradient-norm clipping at 1.0, and a budget of 50 epochs.
Its cosine scheduler uses $T_{\max}=45$ and is stepped once per completed epoch; the scheduler's configuration offset does not implement a separate learning-rate warm-up phase.
Classification is trained with cross-entropy; FSKA additionally uses the spectral penalty weighted by $\lambda_{\mathrm{spec}}=10^{-3}$, with no regularization warm-up in the default configuration.
Validation accuracy is monitored during training, and test accuracy and macro-F1 are evaluated after training.

\paragraph{Models and evaluation.}
We compare Transformer, FlashAttention, Performer, Mamba, RKA, GMM-RKS, KPCA-Scaled, MIKAN, and FSKA.
FSKA uses the positive random features described in Section~\ref{sec:nonstationary_kernel}, with learned query--key projections removed and fixed orthogonal value projections.
Table~\ref{tab:long_seq_cls} reports accuracy and macro-averaged F1 in percent, with mean and standard deviation across three independent runs.
Accuracy measures the fraction of correctly classified documents; macro-F1 averages the per-class F1 scores with equal weight across classes.

\subsection{Additional Ablations on Document Classification}
\label{app:document_ablations}

\subsubsection{Learning the Bivariate Spectral Density}
\label{app:fixed_bivariate}

To separate the contribution of spectral learning from the bivariate feature construction, we replace the learned flow-based spectral density with a fixed bivariate Gaussian.
The bivariate nonstationary feature construction, PRF transformation, removal of learned query--key projections, fixed value projection, and remaining experimental settings are retained.
Table~\ref{tab:fixed_bivariate_doc} reports the comparison on BBC News and 20NG.
The learned spectrum improves accuracy by 26.56 and 8.44 percentage points, and macro-F1 by 26.81 and 8.46 points, respectively.
This comparison supports the contribution of learning the spectrum beyond the fixed bivariate construction evaluated here.

\begin{table}[htbp]
\centering
\caption{Learned versus fixed bivariate spectral density. Accuracy and macro-F1 are reported in percent.}
\label{tab:fixed_bivariate_doc}
\small
\setlength{\tabcolsep}{6pt}
\begin{tabular}{lcccc}
\toprule
& \multicolumn{2}{c}{BBC News} & \multicolumn{2}{c}{20NG} \\
\cmidrule(lr){2-3}\cmidrule(lr){4-5}
Model & Accuracy & Macro-F1 & Accuracy & Macro-F1 \\
\midrule
FSKA & \textbf{90.76} & \textbf{90.31} & \textbf{30.24} & \textbf{28.86} \\
Fixed bivariate Gaussian & 64.20 & 63.50 & 21.80 & 20.40 \\
\bottomrule
\end{tabular}
\end{table}

\subsubsection{Finite-Capacity Query--Key Ablation}
\label{app:qk_capacity}

We compare FSKA with and without learnable query--key projections on 20NG while varying the random-feature capacity and flow depth.
The default flow depth for these document-classification experiments is $N_f=3$.
The first sweep uses $M\in\{32,64,96,128\}$ with $N_f=2$; the second uses $N_f\in\{2,3,5,7,9\}$ with $M=32$.
Table~\ref{tab:qk_capacity_doc} reports the signed accuracy gaps between the two configurations in percentage points.
Following the convention that a positive gap is an accuracy cost of removing projections, we write
\[
\Delta_{\mathrm{QK}}=\operatorname{Acc}(\text{with Q/K})-\operatorname{Acc}(\text{without Q/K}).
\]
Thus, positive values favor retaining learned projections and negative values favor removing them.
Across the tested settings, the gap magnitude is at most 0.90 points.
At $N_f=2$, the gap changes from positive at $M=32,64$ to negative at $M=96,128$; with $M=32$, it is negative at each tested depth of at least three layers.
The results provide a finite-capacity comparison on a document classification task, separate from the approximation assumptions used in the theoretical analysis.

\begin{table}[htbp]
\centering
\caption{Capacity sweeps on 20NG. Entries are signed Q/K accuracy gaps in percentage points.}
\label{tab:qk_capacity_doc}
\small
\setlength{\tabcolsep}{7pt}
\begin{tabular}{lccccc}
\toprule
\multicolumn{6}{l}{Random-feature capacity sweep, fixed $N_f=2$} \\
\midrule
$M$ & 32 & 64 & 96 & 128 & -- \\
Signed gap & $+0.90$ & $+0.81$ & $-0.50$ & $-0.87$ & -- \\
\midrule
\multicolumn{6}{l}{Flow-depth sweep, fixed $M=32$} \\
\midrule
$N_f$ & 2 & 3 & 5 & 7 & 9 \\
Signed gap & $+0.90$ & $-0.51$ & $-0.55$ & $-0.51$ & $-0.47$ \\
\bottomrule
\end{tabular}
\end{table}

\subsubsection{Spectral Regularization Strength}
\label{app:spectral_regularization}

The spectral regularization weight is $\lambda_{\mathrm{spec}}=10^{-3}$ by default.
We vary this weight over $\{0,10^{-4},10^{-3},10^{-2},10^{-1}\}$ on 20NG, as reported in Table~\ref{tab:spec_reg_doc}.
Setting the weight to zero removes the spectral regularization penalty while retaining the learned spectral model.
Relative to the default, removing the penalty reduces accuracy from 30.24\% to 27.42\% and macro-F1 from 28.86\% to 25.83\%.
Increasing the weight to $10^{-1}$ also reduces both metrics.
The default gives the highest reported scores among the tested values; the scores at $10^{-4}$ and $10^{-2}$ are within 0.66 accuracy points and 0.83 macro-F1 points of the default.

\begin{table}[htbp]
\centering
\caption{Spectral regularization sensitivity on 20NG. Accuracy and macro-F1 are reported in percent.}
\label{tab:spec_reg_doc}
\small
\begin{tabular}{lcc}
\toprule
$\lambda_{\mathrm{spec}}$ & Accuracy & Macro-F1 \\
\midrule
0 & 27.42 & 25.83 \\
$10^{-4}$ & 29.58 & 28.03 \\
$10^{-3}$ (default) & \textbf{30.24} & \textbf{28.86} \\
$10^{-2}$ & 29.91 & 28.41 \\
$10^{-1}$ & 28.76 & 27.19 \\
\bottomrule
\end{tabular}
\end{table}

%% file: submit_for_arxiv.bbl
\begin{thebibliography}{22}
\providecommand{\natexlab}[1]{#1}
\providecommand{\url}[1]{\texttt{#1}}
\expandafter\ifx\csname urlstyle\endcsname\relax
  \providecommand{\doi}[1]{doi: #1}\else
  \providecommand{\doi}{doi: \begingroup \urlstyle{rm}\Url}\fi

\bibitem[Brown et~al.(2020)Brown, Mann, Ryder, Subbiah, Kaplan, Dhariwal,
  Neelakantan, Shyam, Sastry, Askell, et~al.]{brown2020language}
Tom Brown, Benjamin Mann, Nick Ryder, Melanie Subbiah, Jared~D Kaplan, Prafulla
  Dhariwal, Arvind Neelakantan, Pranav Shyam, Girish Sastry, Amanda Askell,
  et~al.
\newblock Language models are few-shot learners.
\newblock \emph{Advances in Neural Information Processing Systems},
  33:\penalty0 1877--1901, 2020.

\bibitem[Choromanski et~al.(2021)Choromanski, Likhosherstov, Dohan, Song, Gane,
  Sarlos, Hawkins, Davis, Mohiuddin, Kaiser, et~al.]{choromanski2020rethinking}
Krzysztof Choromanski, Valerii Likhosherstov, David Dohan, Xingyou Song,
  Andreea Gane, Tamas Sarlos, Peter Hawkins, Jared Davis, Afroz Mohiuddin,
  Lukasz Kaiser, et~al.
\newblock Rethinking attention with performers.
\newblock In \emph{International Conference on Learning Representations}, 2021.

\bibitem[Chou et~al.(2024)Chou, Yao, Wang, Pan, Zhu, Zhong, Qiao, Wu, Xu, and
  Li]{chou2024metala}
Yuhong Chou, Man Yao, Kexin Wang, Yuqi Pan, Ruijie Zhu, Yiran Zhong, Yu~Qiao,
  Jibin Wu, Bo~Xu, and Guoqi Li.
\newblock Metala: Unified optimal linear approximation to softmax attention
  map.
\newblock \emph{Advances in Neural Information Processing Systems},
  37:\penalty0 71034--71067, 2024.

\bibitem[Dao et~al.(2022)Dao, Fu, Ermon, Rudra, and
  R{\'e}]{dao2022flashattention}
Tri Dao, Daniel~Y. Fu, Stefano Ermon, Atri Rudra, and Christopher R{\'e}.
\newblock {FlashAttention}: Fast and memory-efficient exact attention with
  {IO}-awareness.
\newblock \emph{Advances in Neural Information Processing Systems},
  35:\penalty0 16344--16359, 2022.
\newblock URL \url{https://arxiv.org/abs/2205.14135}.

\bibitem[Dinh et~al.(2017)Dinh, Sohl-Dickstein, and Bengio]{dinh2016density}
Laurent Dinh, Jascha Sohl-Dickstein, and Samy Bengio.
\newblock Density estimation using real {NVP}.
\newblock \emph{International Conference on Learning Representations}, 2017.

\bibitem[Dosovitskiy et~al.(2021)Dosovitskiy, Beyer, Kolesnikov, Weissenborn,
  Zhai, Unterthiner, Dehghani, Minderer, Heigold, Gelly,
  et~al.]{dosovitskiy2020image}
Alexey Dosovitskiy, Lucas Beyer, Alexander Kolesnikov, Dirk Weissenborn,
  Xiaohua Zhai, Thomas Unterthiner, Mostafa Dehghani, Matthias Minderer, Georg
  Heigold, Sylvain Gelly, et~al.
\newblock An image is worth 16x16 words: Transformers for image recognition at
  scale.
\newblock In \emph{International Conference on Learning Representations}, 2021.

\bibitem[Gu and Dao(2023)]{gu2023mamba}
Albert Gu and Tri Dao.
\newblock {Mamba}: Linear-time sequence modeling with selective state spaces.
\newblock \emph{arXiv preprint arXiv:2312.00752}, 2023.
\newblock URL \url{https://arxiv.org/abs/2312.00752}.

\bibitem[Katharopoulos et~al.(2020)Katharopoulos, Vyas, Pappas, and
  Fleuret]{katharopoulos2020transformers}
Angelos Katharopoulos, Apoorv Vyas, Nikolaos Pappas, and Fran{\c{c}}ois
  Fleuret.
\newblock Transformers are {RNN}s: Fast autoregressive transformers with linear
  attention.
\newblock In \emph{International Conference on Machine Learning}, pages
  5156--5165. PMLR, 2020.

\bibitem[Li et~al.(2019)Li, Chang, Mroueh, Yang, and Poczos]{li2019implicit}
Chun-Liang Li, Wei-Cheng Chang, Youssef Mroueh, Yiming Yang, and Barnabas
  Poczos.
\newblock Implicit kernel learning.
\newblock \emph{Artificial Intelligence and Statistics}, pages 2007--2016,
  2019.

\bibitem[Li and Roth(2002)]{li2002learning}
Xin Li and Dan Roth.
\newblock Learning question classifiers.
\newblock \emph{International Conference on Computational Linguistics}, 2002.

\bibitem[{Pal Chowdhury} et~al.(2022){Pal Chowdhury}, Solomou, Dubey, and
  Sachan]{pal2022learning}
Sankalan {Pal Chowdhury}, Adamos Solomou, Avinava Dubey, and Mrinmaya Sachan.
\newblock Learning the transformer kernel.
\newblock \emph{Transactions on Machine Learning Research}, 2022.

\bibitem[Pang and Lee(2005)]{pang2005seeing}
Bo~Pang and Lillian Lee.
\newblock Seeing stars: Exploiting class relationships for sentiment
  categorization with respect to rating scales.
\newblock \emph{Association for Computational Linguistics}, pages 115--124,
  2005.

\bibitem[Peng et~al.(2021)Peng, Pappas, Yogatama, Schwartz, Smith, and
  Kong]{peng2021random}
Hao Peng, Nikolaos Pappas, Dani Yogatama, Roy Schwartz, Noah Smith, and
  Lingpeng Kong.
\newblock Random feature attention.
\newblock In \emph{International Conference on Learning Representations}, 2021.

\bibitem[Rahimi and Recht(2007)]{rahimi2007random}
Ali Rahimi and Benjamin Recht.
\newblock Random features for large-scale kernel machines.
\newblock \emph{Advances in Neural Information Processing Systems}, 20, 2007.

\bibitem[Socher et~al.(2013)Socher, Perelygin, Wu, Chuang, Manning, Ng, and
  Potts]{socher2013recursive}
Richard Socher, Alex Perelygin, Jean Wu, Jason Chuang, Christopher~D Manning,
  Andrew~Y Ng, and Christopher Potts.
\newblock Recursive deep models for semantic compositionality over a sentiment
  treebank.
\newblock \emph{Empirical Methods in Natural Language Processing}, pages
  1631--1642, 2013.

\bibitem[Song et~al.(2021)Song, Jung, Kim, and Moon]{song2021implicit}
Kyungwoo Song, Yohan Jung, Dongjun Kim, and Il-Chul Moon.
\newblock Implicit kernel attention.
\newblock In \emph{Proceedings of the AAAI Conference on Artificial
  Intelligence}, volume~35, pages 9713--9721, 2021.

\bibitem[Tay et~al.(2020)Tay, Dehghani, Abnar, Shen, Bahri, Pham, Rao, Yang,
  Ruder, and Metzler]{tay2020long}
Yi~Tay, Mostafa Dehghani, Samira Abnar, Yikang Shen, Dara Bahri, Philip Pham,
  Jinfeng Rao, Liu Yang, Sebastian Ruder, and Donald Metzler.
\newblock Long range arena: A benchmark for efficient transformers.
\newblock \emph{arXiv preprint arXiv:2011.04006}, 2020.

\bibitem[Teo and Nguyen(2024)]{teo2024unveiling}
Rachel~S Teo and Tan~M Nguyen.
\newblock Unveiling the hidden structure of self-attention via kernel principal
  component analysis.
\newblock \emph{Advances in Neural Information Processing Systems},
  37:\penalty0 101393--101427, 2024.

\bibitem[Tsai et~al.(2019)Tsai, Bai, Yamada, Morency, and
  Salakhutdinov]{tsai2019transformer}
Yao-Hung~Hubert Tsai, Shaojie Bai, Makoto Yamada, Louis-Philippe Morency, and
  Ruslan Salakhutdinov.
\newblock Transformer dissection: {A} unified understanding for
  {T}ransformer’s attention via the lens of kernel.
\newblock 2019.

\bibitem[Vaswani et~al.(2017)Vaswani, Shazeer, Parmar, Uszkoreit, Jones, Gomez,
  Kaiser, and Polosukhin]{vaswani2017attention}
Ashish Vaswani, Noam Shazeer, Niki Parmar, Jakob Uszkoreit, Llion Jones,
  Aidan~N Gomez, {\L}ukasz Kaiser, and Illia Polosukhin.
\newblock Attention is all you need.
\newblock \emph{Advances in Neural Information Processing Systems}, 30, 2017.

\bibitem[Yang et~al.(2025)Yang, Li, Lin, Zhang, and Olmos]{yang2025multi}
Zi~Yang, Ying Li, Zhidi Lin, Michael~Minyi Zhang, and Pablo~M Olmos.
\newblock Multi-view oriented {GPLVM}: Expressiveness and efficiency.
\newblock In \emph{Advances in Neural Information Processing Systems}, 2025.

\bibitem[Zhang et~al.(2015)Zhang, Zhao, and LeCun]{zhang2015character}
Xiang Zhang, Junbo Zhao, and Yann LeCun.
\newblock Character-level convolutional networks for text classification.
\newblock In \emph{Advances in Neural Information Processing Systems},
  volume~28, 2015.

\end{thebibliography}
